\documentclass[lettersize,journal]{IEEEtran}
\usepackage{cite}
\usepackage{amsmath,amssymb,amsfonts}
\usepackage[linesnumbered,ruled,vlined]{algorithm2e}
\usepackage{graphicx}
\usepackage{textcomp}
\usepackage{xcolor}
\usepackage{makecell}
\usepackage{adjustbox}
\usepackage{subfigure}
\usepackage{multirow}
\usepackage{lipsum}

\newtheorem{proposition}{Proposition}
\newtheorem{remark}{Remark}

\DeclareMathOperator*{\argmin}{arg\,min}
\begin{document}

\title{
Variations of Augmented Lagrangian for \\ Robotic Multi-Contact Simulation
}
\author{Jeongmin Lee, Minji Lee, Sunkyung Park, Jinhee Yun, and Dongjun Lee
\thanks{This research was supported by the National Research Foundation (NRF) grant funded by the Ministry of Trade, Industry \& Energy (MOTIE) of Korea (RS-2024-00419641), and the Ministry of Science and ICT (MSIT) of Korea (RS-2022-00144468). Corresponding author: Dongjun Lee.
}
\thanks{The authors are with the Department of Mechanical Engineering, IAMD and IOER, Seoul National University, Seoul, Republic of Korea.
\{ljmlgh,mingg8,sunk1136,yjhs0932,djlee\}@snu.ac.kr. }
}
\maketitle

\begin{abstract}
The multi-contact nonlinear complementarity problem (NCP) is a naturally arising challenge in robotic simulations.
Achieving high performance in terms of both accuracy and efficiency remains a significant challenge, particularly in scenarios involving intensive contacts and stiff interactions. 
In this article, we introduce a new class of multi-contact NCP solvers based on the theory of the Augmented Lagrangian (AL). We detail how the standard derivation of AL in convex optimization can be adapted to handle multi-contact NCP through the iteration of surrogate problem solutions and the subsequent update of primal-dual variables. 
Specifically, we present two tailored variations of AL for robotic simulations: the Cascaded Newton-based Augmented Lagrangian (CANAL) and the Subsystem-based Alternating Direction Method of Multipliers (SubADMM). We demonstrate how CANAL can manage multi-contact NCP in an accurate and robust manner, while SubADMM offers superior computational speed, scalability, and parallelizability for high degrees-of-freedom multibody systems with numerous contacts. Our results showcase the effectiveness of the proposed solver framework, illustrating its advantages in various robotic manipulation scenarios.
\end{abstract}


\section{Introduction}

\IEEEPARstart{P}{h}ysics simulation is a fundamental tool for the development of robotic intelligence, as it enables scalable data acquisition, training, and safe testing of various algorithms and designs. 
Moreover, simulations can be directly employed to solve modeled system dynamics, proving invaluable for a range of applications such as global planning, trajectory optimization, and parameter estimation. 
This significance has led to the development of diverse open-source platforms \cite{bullet,mujoco,physx,brax,RaiSim,sofa}, which are increasingly being utilized in various research endeavors.

An essential focus in robotic simulation research revolves around achieving results that are both accurate and efficient in terms of memory and computation time. This presents a comprehensive and challenging problem, encompassing diverse considerations such as discrete-time integration, defining various geometric/physical constraints, incorporating friction, managing system-induced sparsity, and selecting numerical algorithms. 
Among these factors, multi-contact plays a crucial role in mimicking interactions between objects. A prevalent velocity-level modeling of such constraints \cite{stewart1996implicit} naturally induces a nonlinear complementarity problem (NCP), which is generally challenging to solve.

Typically, contact solvers for physics simulations must balance three crucial factors: efficiency, accuracy, and robustness. However, finding a universal solution remains challenging. 
Methods developed for graphics and game engines tend to prioritize efficiency and robustness, aiming to deliver visually plausible results, even if early termination occurs. However, they are known to converge slowly and may struggle with achieving highly accurate solutions. They frequently encounter difficulties in handling intensive contact interactions (i.e., where constraints are dense and numerous relative to the system degrees of freedom), which is common in robotic manipulation. 
Conversely, achieving a highly accurate solution for NCP often involves complex matrix operations and numerically sensitive processes, which generally lack efficiency and robustness for practical robotic applications. 
Moreover, some approaches aim to enhance efficiency and robustness by relaxing the contact constraints and exploiting them during the solving stage. However, such relaxations can be challenging to physically interpret, and the solutions they produce may exhibit undesirable physical behaviors.

\begin{figure}[t]
\centering
\includegraphics[width=8cm]{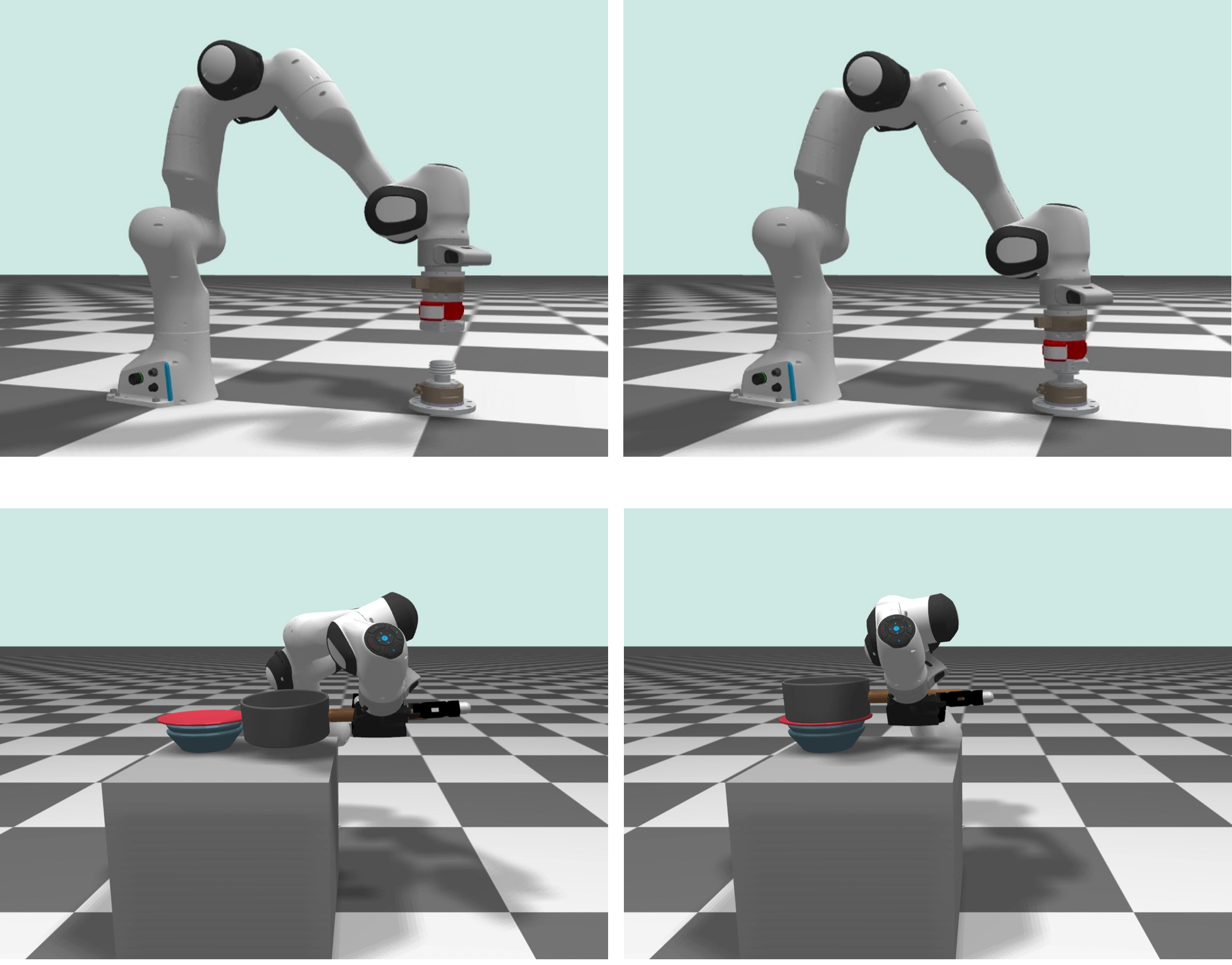}
\caption{Snapshots of a robotic simulation using our multi-contact solver. Top: bolt-nut assembly. Bottom: dish piling. Although intensive contact formation and stiff interactions make these scenarios challenging to simulate, our solvers successfully complete the simulations less than a $\rm{ms}$ of time budget per step.}
\label{fig:thumbnail}
\end{figure}

In this article, we introduce a new series of multi-contact solvers for robotic simulation based on the theory of augmented Lagrangian (AL). 
We demonstrate how the variations of AL can address the multi-contact NCP for robotic simulations, by iteratively solving surrogate problems, thereby enabling the proximal solution converges in a stable and robust manner.
Specifically, we present two algorithms that are practically applicable to robotic simulation: the cascaded Newton-based Augmented Lagrangian method (CANAL) and the subsystem-based Alternating Direction Method of Multipliers (SubADMM).
We explain how these two variations are advantageous in scenarios requiring precise management of high-density intensive contact and parallelized, scalable handling of high degree of freedom (DOF) multibody contact, respectively.
Several robotic simulations, particularly those involving challenging multi-contact scenarios, are implemented and demonstrated to validate our framework.

The rest of the article is structured as follows. 
In Sec.~\ref{sec:relatedworks} we review the development and utilization of multi-contact solvers in robotic applications and beyond.
Sec.~\ref{sec:background} provides essential background materials necessary to present our AL-based multi-contact solver.
Then, Sec.~\ref{sec:multicontactAL} presents our core theories and structures for the AL-based multi-contact solver. This leads to Sec.~\ref{sec:canal}, which outlines the first practical variation as the cascaded Newton-based AL, and Sec.~\ref{sec:subadmm}, which introduces the other variation: subsystem-based ADMM.
Sec.~\ref{sec:exampleeval} illustrates the implementation results of our solver in physics simulation and evaluates its performance under various robotic manipulation scenarios.
Finally, Sec.~\ref{sec:discussremark} conclude the article with discussions and remarks.

\section{Related Works} \label{sec:relatedworks}

In this section, we summarize the multi-contact modeling and solver algorithms that have been utilized in robotic simulation. See also Table~\ref{table:simulator_compare} for the comparison of widely used simulators in robotics.

\subsection{Direct Method}
The conventional approach to handling dynamics equations with multi-contact constraints involves formulating the equations as a linear complementarity problem (LCP) \cite{potra1997formulating} then applying Lemke's algorithms \cite{llyod2005fast} or Dantzig's pivoting algorithms. 
While these direct methods can guarantee accuracy, they often suffer from high computational complexity. Moreover, the LCP-based formulation necessitates polygonal friction cone approximation, leading to undesirable error in friction behavior. 
In robotic simulation software, DART \cite{lee2018dart}, ODE \cite{ode}, and Bullet \cite{bullet} provide implementations of Dantzig's method to solve the LCP problem.

\subsection{Per-Contact Iteration}

More widely used in recent years are iterative methods, which typically involve locally performing an impulse projection step to achieve global equilibrium. 
One of the most popular iteration schemes is projected Gauss-Seidel (PGS), which has been extensively developed and adopted in the game and graphics community \cite{macklin2014unified,macklin2016xpbd} as well as in robotics \cite{todorov2014convex,horak2019similarities}. 
These methods are known for being simple, robust, and advantageous in generating visually plausible results. 
However, they often experience slow convergence and limited efficiency, especially when the constraints are highly coupled. 
These weaknesses are particularly emphasized in robotic simulation, as the generalized coordinate representation (e.g., robot joint angles) is common, and over-specified contact (i.e., system DOF $<$ constraint DOF) is prevalent in manipulation tasks.
Several research efforts have aimed to enhance the performance of impulse iteration methods. In \cite{hwangbo2018per}, the bisection method is presented as a potential replacement for the local projection scheme in PGS, demonstrating its effectiveness in quadruped locomotion simulation. Additionally, a substepping variant of PGS, named temporal Gauss-Seidel (TGS), is introduced in \cite{macklin2019small}, showing its better convergence in various situations.
Unlike direct methods, iterative methods can be applied to various types of problem modeling, including LCP, cone complementarity problems (CCP), nonlinear complementarity problems (NCP), and also their position-based dynamics (PBD) variants \cite{muller2020detailed}. 
As a result, they are employed in a wide range of simulation software, including Bullet \cite{bullet}, MuJoCo \cite{mujoco}, RaiSim \cite{RaiSim}, and Isaac Sim \cite{makoviychuk2021isaac}.

\subsection{Nonlinear Equation}

Another approach to dealing with multi-contact simulation is to express all required relations in nonlinear equation form and solve them using gradient descent iteration. 
Implicit penalty-based contact, often referred to as regularized contact, exhibits the most natural connection to this approach, as demonstrated in \cite{geilinger2020add,castro2020transition}. 
However, penalty methods have well-known weaknesses that they often necessitate parameter tuning to achieve plausible results, and high penalty gains can lead to numerical issues.
For the other direction, in \cite{howell2022dojo} construct and solve a nonlinear equation with complementarity smoothing, and \cite{macklin2019nonsmooth} we derived a nonsmooth equation using the complementarity function (e.g., Fischer-Burmeister).
While these methods typically exhibit superlinear convergence, the intricate nature of contact conditions frequently leads to lack of robustness or challenges in line search. 
Addressing this issue, the Newton-based techniques \cite{castro2022unconstrained} and conjugate gradient (CG) algorithm for regularized convex contact models aim to ensure algorithmic robustness, albeit at the potential expense of physical accuracy. 
Among current simulation software, MuJoCo and Drake \cite{drake} are incorporating nonlinear equation-based solvers.

\begin{table}[t]
\centering
\caption{Comparison of contact models and solvers used in popular robotic simulators.}
\renewcommand{\arraystretch}{2.0}{
\resizebox{8.0cm}{!}{
\begin{tabular}{|c|c|c|c|c|c|c|}
\hline
& Bullet & MuJoCo & DART & PhysX & Drake & ODE \\
\hline
\hline
Model & LCP & Convex & LCP & NCP & Convex & LCP \\
\hline 
Solver & \makecell{Direct \\ PGS} & \makecell{Newton \\ CG \\ PGS} & \makecell{Direct \\ PGS} & \makecell{PGS \\ TGS} & Newton & \makecell{Direct \\ PGS} \\
\hline 
\end{tabular}
}   
}
\label{table:simulator_compare}
\end{table}

\subsection{Augmented Lagrangian}
 
Proximal algorithms, which were possibly pioneered by Moreau \cite{moreau1962fonctions} comprise a class of methods designed to address constrained convex optimization problems by sequentially solving a series of subproblems. 
The augmented Lagrangian (AL) method can be viewed as a class of proximal algorithm \cite{parikh2014proximal}, as it formulates subproblems using the method of multipliers and a penalty term. 
Typically, the subproblems are addressed through simpler solutions or tailored designs, which has spurred the development of numerous open-source libraries that implement these strategies, thereby facilitating broader access to robust optimization tools. Notable examples include libraries for quadratic programming, such as OSQP \cite{stellato2020osqp} and QPALM \cite{hermans2019qpalm}, and for cone programming, such as SCS \cite{sopasakis2019superscs}.
In robotics, proximal algorithms have been effectively utilized to address constraints within computational structures, notably in applications such as factor graph optimization \cite{bazzana2024augmented} and differentiable dynamics programming \cite{howell2019altro}.
The utility in solving robot dynamics with equality constraints is presented in \cite{carpentier2021proximal}. 
Our previous work \cite{lee2023modular} presents a specific algorithm based on the Augmented Lagrangian (AL) method to achieve effective parallelization in contact simulations. Building upon this foundation, this article extends the general theory and variations of AL designed to handle robotic simulations involving contact.

\section{Preliminary} \label{sec:background}

\subsection{Discretized Dynamics}
We consider following continuous-time equations of motion:
\begin{align} \label{eq:continuous_dyn}
    M(q)\ddot{q} = f(q,\dot{q}) + J(q)^T\lambda
\end{align}
where $q\in\mathbb{R}^n$ is the generalized coordinate variable of system, $M(q)\in\mathbb{R}^{n\times n}$ is the system mass matrix, $f\in\mathbb{R}^n$ is the generalized force (including Coriolis/gravitational force, external input, etc.) and $\lambda\in\mathbb{R}^{n_c}, J(q)\in\mathbb{R}^{n_c\times n}$ are the constraint impulse and Jacobian with $n,n_c$ being the system/constraint dimension.
In typical robotic simulation, the discretized version of the equation \eqref{eq:continuous_dyn} is employed:
\begin{align} \label{eq:discrete_dyn}
\begin{split}
    &M_k(v_{k+1}-v_k) = f_k t_k+ J_{k}^T\lambda_{k} \\
    &\hat{v}_k=\theta v_k+ (1-\theta)v_{k+1} \\
    &q_{k+1} \leftarrow \text{update}(q_k, \hat{v}_k, t_k)
\end{split}
\end{align}
where $k$ denotes the time step index, $t_k$ is the step size, and the $v_k$ is the generalized velocity at the $k$-th step. 
In this work, we primarily integrate explicit and implicit schemes. Specifically, we utilize $M_k = M(q_k)$ and $f_k = f(q_k, v_k)$, while employing the representative mid-step velocity $\hat{v}_k \in \mathbb{R}^n$ for state updates and constraint handling.
Here, $\theta \in [0, 1]$ determines the precise integration rule, while its impact on physical behavior is discussed in \cite{kim2017haptic}.
From now on, time step index $k$ will be omitted for simplicity but note that all components are still time(step)-varying.

\subsection{Constraint Models}

Throughout this article, we classify the constraints on the multibody system into three categories: soft, hard, and contact constraints.
Similar to many other simulators \cite{RaiSim, bullet, mujoco}, the constraint model can be formulated by relation between the velocity $\hat{v}$ to the impulse $\lambda$.
Such velocity-impulse modeling has advantages in terms of the well-definedness of the problem (c.f., the Painleve paradox \cite{acary2011formulation}) and can naturally express behaviors like friction or elastic collisions. However, it may exhibit position-level drift, as it is based on linearization on the constraints.
Positional drift can be suppressed by adopting techniques such as multiple linearization, as in \cite{daviet2020simple}, or re-linearization \cite{verschoor2019efficient}, during the solution process. These methods may be considered for future implementation.

\subsubsection{Hard Constraint}

Hard constraints ensure that equations and inequalities for the system are strictly satisfied (e.g., joint limit), including holonomic and non-holonomic types. If the $i$-th constraint is hard, the corresponding relation is
\begin{align} \label{eq:hardcon}
    0 \le \lambda_i \perp J_i\hat{v} + e_i \ge 0
\end{align}
where $e_i\in\mathbb{R}$ and $J_i\in\mathbb{R}^{1\times n}$ denote the error and Jacobian for hard constraint.
Here, the error $e_i$ is typically scaled and biased value, from the methods such as Baumgarte stabilization \cite{baumgarte1972stabilization}, to prevent the constraint drift effectively.

\subsubsection{Soft Constraint}
Soft constraints are typically originated from the elastic potential energy of the system (e.g., spring).
If the $i$-th constraint is soft, constraint impulse can be written as
\begin{align} \label{eq:softcon}
    \lambda_i = -k_i e_i - b_i J_i\hat{v}
\end{align}
where $e_i\in\mathbb{R}$ and $J_i\in\mathbb{R}^{1\times n}$ are the error and Jacobian for soft constraint, $k_i,b_i>0$ are the gain and damping parameter which are scaled and biased dependent on the time integration scheme, step size, and constraint-space damping. 
The values of $k_i$ and $b_i$ are associated with the system energy behavior, see \cite{andrews2017geometric,kim2017haptic} for more details.

\subsubsection{Contact Constraint}

\begin{figure}[t]
\centering
\includegraphics[width=8cm]{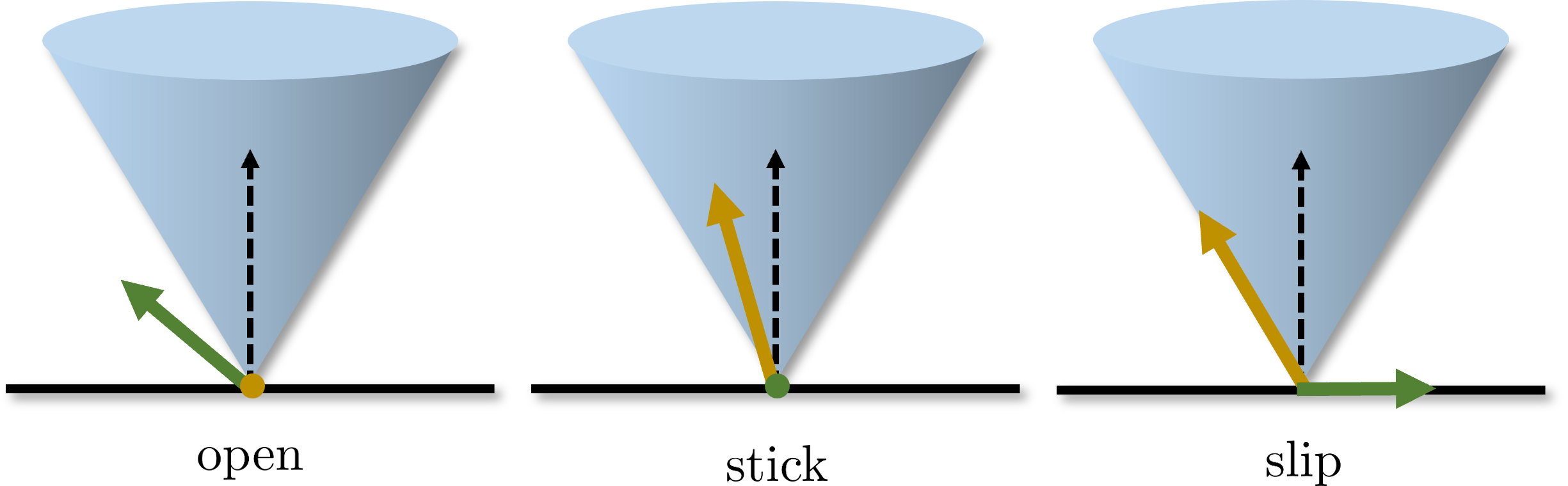}
\caption{Three cases resulting from the Signorini-Coulomb condition, ranging from open ($\lambda_{i,n}=0$), stick ($\lambda_{i,n}>0, \delta_i=0$), to slip ($\lambda_{i,n}>0, \delta_i>0$), shown from left to right. 
The blue shape illustrates the friction cone, the green arrow indicates the contact frame velocity, and the yellow arrow represents the contact impulse.}
\label{fig:signorini_coulomb}
\end{figure}

Contact condition is typically the most demanding type since it includes nonlinear complementarity relation between primal (i.e., velocity) and dual (i.e., impulse) variables.
In this article, we assume that at each time step, set of contact features (i.e., gap, contact point, normal) are provided from collision detection module \cite{pan2012fcl,lee2023differentiable}.
Then for each contact point, we define Signorini-Coulomb condition (SCC), which is the most universal expression for dry frictional contact.
If the $i$-th constraint is contact, the corresponding 3-DOF relation is
\begin{align} \label{eq:contactcon}
    \begin{split}
    & 0 \le \lambda_{i,n} \perp J_{i,n}\hat{v} + e_{i,n} \ge 0 \\
    & 0 \le \delta_i \perp \mu_i\lambda_{i,n} - \| \lambda_{i,t} \| \ge 0 \\
    &\delta_i \lambda_{i,t} + \mu_i\lambda_{i,n} J_{i,t}\hat{v} = 0
    \end{split}
\end{align} 
where $\perp$ denotes complementarity, $e_{i,n}\in\mathbb{R}$ and $J_{i,n}\in\mathbb{R}^{1\times n}$ denote the error and Jacobian for contact normal, $J_{i,t}\in\mathbb{R}^{2\times n}$ is the Jacobian for contact tangential, and $\mu_i$ is the friction coefficient and $\delta_i$ is the auxiliary variable.
The first condition, known as the velocity-level Signorini condition, captures the complementarity nature of the contact occurence and gap. 
The remaining conditions involve the complementarity between slipping velocity and the friction cone boundary, with the maximal dissipation law indicating that slip opposes the direction of impulse.
There are three situations induced by the condition \eqref{eq:contactcon} - open, stick, and slip, as depicted in Fig.~\ref{fig:signorini_coulomb}.

\subsection{Augmented Lagrangian Method}

By standard, augmented Lagrangian (AL) method is a class of algorithms to solve the following constrained optimization problem:
\begin{align*}
    \min_{x,z} f(x) + g(z) \quad \text{s.t.} \quad Px+Qz = r.
\end{align*}
Here, the augmented Lagrangian is defined as,
\begin{align*} 
    &\mathcal{L} = f(x) + g(z) + u^T(Px+Qz-r) + \frac{\beta}{2} \| Px+Qz - r\|^2
\end{align*}
where $u$ is the Lagrange multiplier and $\beta > 0$ is the penalty weight.
Then AL method takes the iteration step as
\begin{align} \label{eq:ALM}
\begin{split}
    (x^{l+1},z^{l+1}) &= \argmin_{x,z} \mathcal{L}(x,z,u^l) \\
    u^{l+1} &= u^l + \beta(Px^{l+1}+Qz^{l+1}-r)
\end{split}
\end{align}
where $l$ is the iteration index. 
In the above \eqref{eq:ALM}, $(x,z)$ are coupled for minimization problem at each step. 
Meanwhile, Alternating direction method of multiplier (ADMM \cite{boyd2011distributed}) iteratively performs alternating minimization of $\mathcal{L}$ with respect to each variable. The iteration process of ADMM can be summarized as follow:
\begin{align} \label{eq:ADMM}
\begin{split}
    &x^{l+1} = \argmin_{x} \mathcal{L}(x,z^l,u^l)  \\
    &z^{l+1} = \argmin_{z} \mathcal{L}(x^{l+1},z,u^l)  \\
    &u^{l+1} = u^l + \beta(Px^{l+1}+Qz^{l+1}-r)
\end{split}
\end{align}
where $l$ is the loop index. 
By independently resolving each variable, ADMM is often employed to enhance the efficiency and scalability of the application \cite{boyd2011distributed, wang2019admm}. 
In this work, we develop separate tailored algorithms based on the styles of \eqref{eq:ALM} and \eqref{eq:ADMM}.

\section{Multi-Contact Simulation via Augmented Lagrangian} \label{sec:multicontactAL}

\subsection{Problem Formulation}

The motivation for leveraging AL in contact simulation primarily stems from the insight to integrate tools from constrained optimization into the solving of constrained dynamics equation. 
The problem considered in this article, can be essentially formulated as 
\begin{align} \label{eq:prob_original}
\begin{split}
    &\text{Solve}~A\hat{v} = b + J^T\lambda \\
    &~~~~\text{s.t.}~~(J\hat{v},\lambda)\in \mathcal{S}_c
\end{split}
\end{align}
where $A\in\mathbb{R}^{n\times n}, b\in\mathbb{R}^n$ are the dynamics matrix/vector compressed from \eqref{eq:discrete_dyn}, and $\mathcal{S}_c$ represent the set that satisfies the relation between $J\hat{v}$ and $\lambda$ described in \eqref{eq:hardcon}, \eqref{eq:softcon}, and \eqref{eq:contactcon}. 
In robotic manipulation scenarios, contact points are often generated numerously and densely, which can lead to the resulting problem being ill-conditioned or infeasibly defined. 
In such cases, performing per-contact iteration based on the dual conversion $JA^{-1}J^T$ (i.e., so-called Delassus operator) often proves inefficient and exhibits slow convergence.
Meanwhile, AL in optimization is known for maintaining subproblem feasibility and demonstrating robust convergence, even converging to solutions with the least constraint violation in poorly defined problems \cite{dai2023augmented}.

Consequently, our primary objective is to investigate whether the advantages of the AL approach can be effectively applied to contact simulation. 
Although the problem \eqref{eq:prob_original} shares commonalities with optimization, it diverges due to the introduction of complementarity relations between primal and dual variables, particuarly associated with contact conditions. 
Our aim is to establish a foundation for deriving AL techniques specifically tailored to multi-contact, thereby addressing the unique challenges posed in robotic simulation.

\subsection{AL for Multi-Contact NCP} \label{subsec:AL_derivation}

In this subsection, we will derive augmented Lagrangian (AL) based methods to address multi-contact NCP in simulation.
We start by equivalently expressing \eqref{eq:prob_original} as follows:
\begin{align} \label{eq:prob_original_slack}
\begin{split}
    &\text{Solve}~A\hat{v} = b + J^T\lambda\\
    &~~~~\text{s.t.}~~(z,\lambda)\in \mathcal{S}_c,~J\hat{v} = z
\end{split}
\end{align}
where $z\in\mathbb{R}^{n_c}$ serves as the slack variable for the constraint interface. The expression in \eqref{eq:prob_original_slack} bears resemblance to the optimality condition of the following optimization problem:
\begin{align} \label{eq:prob_original_opt}
&\min_{\hat{v},z}~\frac{1}{2}\hat{v}^TA\hat{v} - b^T\hat{v} + g(z)~~\text{subject to}~~J\hat{v} = z
\end{align}
as the matrix $A$ is always symmetric positive definite. In this context, $g$ serves to enforce the constraint in dynamics, although $(z,\lambda)\in\mathcal{S}_c$ is not integrable into the function if the multi-contact condition included.
Recalling the structure of \eqref{eq:ALM} applicable to the optimization problem \eqref{eq:prob_original_opt}, we can similarly solve \eqref{eq:prob_original_slack} as follows:
\begin{align} 
\begin{split}
&\text{Solve }
\begin{bmatrix}
A+\beta J^TJ & -\beta J^T \\
-\beta J & \beta I
\end{bmatrix}
\begin{bmatrix}
\hat{v} \\ z
\end{bmatrix}
=
\begin{bmatrix}
b-J^Tu \\ u + \lambda
\end{bmatrix} \\
&~~~~\text{s.t.}~~
(z,\lambda)\in\mathcal{S}_c \label{eq:prob_surrogate}
\end{split}
\\
&u \leftarrow u + \beta(J\hat{v}-z). \label{eq:AL_multiplier_update}
\end{align}

The rationale of the above structure is that, at the fixed-point of the iteration (therefore, $J\hat{v}=z$), the result satisfies both dynamics equation and constraint relation.
Similar to \eqref{eq:ALM}, the process can be interpreted as iterating between solving the problem relaxed via a penalty term and updating the Lagrange multipliers. 
We refer this relaxed problem \eqref{eq:prob_surrogate} as the \textit{surrogate} problem.
However, unlike the minimization problem in \eqref{eq:ALM}, the solvability of the surrogate problem remains unclear, which may raise potential concerns. We address this issue in the following proposition.
\begin{proposition} \label{prop:surrogate_feasibility}
Surrogate problem \eqref{eq:prob_surrogate} always has a feasible solution. 
\end{proposition}
\begin{proof}
    The proof can be done by borrowing the existence proof from \cite{acary2011formulation} based on the Brouwer fixed-point theorem, which states that a solution to the Coulomb friction problem always exists if the rows of the contact Jacobian are linearly independent. Due to the slack variable $z$, the contact Jacobian in \eqref{eq:prob_surrogate} can be simply considered as an identity matrix, satisfying this condition.
\end{proof}

Given that the surrogate problem \eqref{eq:prob_surrogate} has a feasible solution, the numerical scheme used to find this solution becomes significant as its performance directly affects the overall efficiency and accuracy of the AL methods for multi-contact NCP.

\subsection{Closed-Form Formulation of Slack Variables}
\label{subsec:closed_form_slack}

Compared to the original problem \eqref{eq:prob_original}, the surrogate problem \eqref{eq:prob_surrogate} should be easier to solve in order to maintain the rationality of the framework. 
A crucial difference between \eqref{eq:prob_original} and \eqref{eq:prob_surrogate} is that the constraint condition is defined on the slack variable $z$ as shown below:
\begin{align} \label{eq:rel_z_lambda}
\beta z = \beta J\hat{v} + u + \lambda, \quad \text{s.t.} \quad (z,\lambda) \in \mathcal{S}_c.
\end{align}
This implies that the relationship between $z$ and $\lambda$ is matrix-free and involves only a simple scalar weight $\beta$.
Based on this feature, we can derive the \textit{closed-form} representation for $\lambda$ (therefore, also for $z$) with respect to $\hat{v}$ for each hard, soft, and contact constraint. 
The derivations are listed as below.

If the $i$-th constraint is hard, we can determine $\lambda_i$ by substituting \eqref{eq:rel_z_lambda} into the complementarity relation \eqref{eq:hardcon}: 
\begin{align} \label{eq:closedform_hard}
\begin{split}
    &\lambda_i = \Pi_{\ge 0}\left (-\beta J_i\hat{v}-u_i-\beta e_i \right) 
\end{split}
\end{align}
where $\Pi_{\ge 0}$ denotes the projection on positive set.
Meanwhile if the $i$-th constraint is soft, we can determine $\lambda_i$ by substituting \eqref{eq:rel_z_lambda} into the linear relation \eqref{eq:softcon}: 
\begin{align} \label{eq:closedform_soft}
\begin{split}
    &\lambda_i = -\frac{b_i(\beta J_i\hat{v}+u_i) + \beta k_i e_i}{b_i+\beta}
\end{split}
\end{align}
Finally, if the $i$-th constraint represents a contact, we can determine $\lambda_i$ by substituting \eqref{eq:rel_z_lambda} into the contact condition described in \eqref{eq:contactcon}:
\begin{align} \label{eq:closedform_contact}
&\lambda_i = \Pi_{\mathcal{C}}^\text{strict} \left(-\beta J_i\hat{v}-u_i-\beta e_i \right)
\end{align}
where $\Pi_{\mathcal{C}}$ denotes the projection onto the friction cone $\mathcal{C}$. 
Specifically, the projection $\lambda_i = \Pi_{\mathcal{C}}^\text{strict}\left(\lambda_i^*\right)$ is carried out by the following steps:
\begin{align} \label{eq:strict_operator}
\begin{split}
    \lambda_{i,n} &= \max(\lambda_{i,n}^{*},0) \\ 
    \lambda_{i,t} &= \Pi_{\mathcal{C}(\lambda_{i,n})}(\lambda_{i,t}^{*}). 
\end{split}
\end{align}
Here, $\mathcal{C}(\lambda_{i,n})$ represents the cross-section of $\mathcal{C}$ where the plane at height $\lambda_{i,n}$ intersects the cone.
This nested projection is distinct from the closest distance projection, commonly known as the proximal operator when applied to the indicator function of the friction cone \cite{parikh2014proximal}. 
See Fig.~\ref{fig:strict_proximal} for illustrations of each projection scheme.
As in \cite{lee2022large}, we refer to \eqref{eq:strict_operator} as the strict operator as the resulting $(z_i,\lambda_i)$ strictly satisfies the Signorini-Coulomb condition. This property is demonstrated in the proposition below.
\begin{proposition} \label{prop:strict_scc}
If the $i$-th constraint represents a contact, \eqref{eq:closedform_contact} with \eqref{eq:strict_operator} gives the unique solution of \eqref{eq:rel_z_lambda}.
\end{proposition}
\begin{proof}
As the normal component is completely decoupled from the tangential component in \eqref{eq:rel_z_lambda}, it can be written as 
\begin{align*}
\begin{split}
    \beta (z_{i,n} + e_{i,n}) &= \lambda_{i,n} + \beta J_{i,n}\hat{v} + u_{i,n} + \beta e_{i,n} \\   
    &= \lambda_{i,n} - \lambda_{i,n}^*
\end{split}
\end{align*}
Then, if $\lambda_{i,n}^{*}>0$, $\lambda_{i,n}=\lambda_{i,n}^{*}$ is the only solution for which $z_{i,n}+e_{i,n}=0$ is satisfied. 
Otherwise, if $\lambda_{i,n}=0$, it is the only solution since $z_{i,n}+e_{i,n}>0$.
Therefore, the normal components are uniquely determined, satisfying the complementarity condition.
For the tangential components, the relation can be written as
\begin{align*}
    \beta z_{i,t} = \lambda_{i,t} - \lambda_{i,t}^*
\end{align*}
Substituting the above equation into \eqref{eq:contactcon}, we obtain
\begin{align*}
    (\beta\delta_i + \mu_i\lambda_{i,n})\lambda_{i,t} = \mu_i\lambda_{i,n}\lambda_{i,t}^*
\end{align*} 
If $\lambda_{i,t}$ lies inside $\mathcal{C}(\lambda_{i,n})$, $\mu\lambda_{i,n}-\|\lambda_{i,t}\|>0$ holds and $\delta_i$ should be $0$.
If $\lambda_{i,t}$ lies outside $\mathcal{C}(\lambda_{i,n})$, $\delta_i$ should be larger than $0$, yet should satisfy $\mu\lambda_{i,n}=\|\lambda_{i,t}\|$, therefore it is uniquely determined.
Finally, the resulting $\lambda_{t,i}$ is equivalent to the result of the strict operator, thus the statement holds.
\end{proof}

The results \eqref{eq:closedform_hard}, \eqref{eq:closedform_soft}, and \eqref{eq:closedform_contact} derived above imply that $\lambda_i$ can be expressed as a closed-form regardless of constraint type, allowing us to write it as:
\begin{align} \label{eq:closedform_operator}
\lambda_i = T(\lambda_i^*) \quad \text{where} \quad \lambda_i^* = -\beta J_i\hat{v}-u_i -\beta e_i
\end{align}
where $T$ is a closed-form operator which is continuous yet may nonsmooth depending on the constraint type.
Accordingly, by the linear relation \eqref{eq:rel_z_lambda}, the slack variable $z$ is also expressed in closed-form with respect to $\hat{v}$.

\begin{figure}[t]
\centering
\includegraphics[width=6cm]{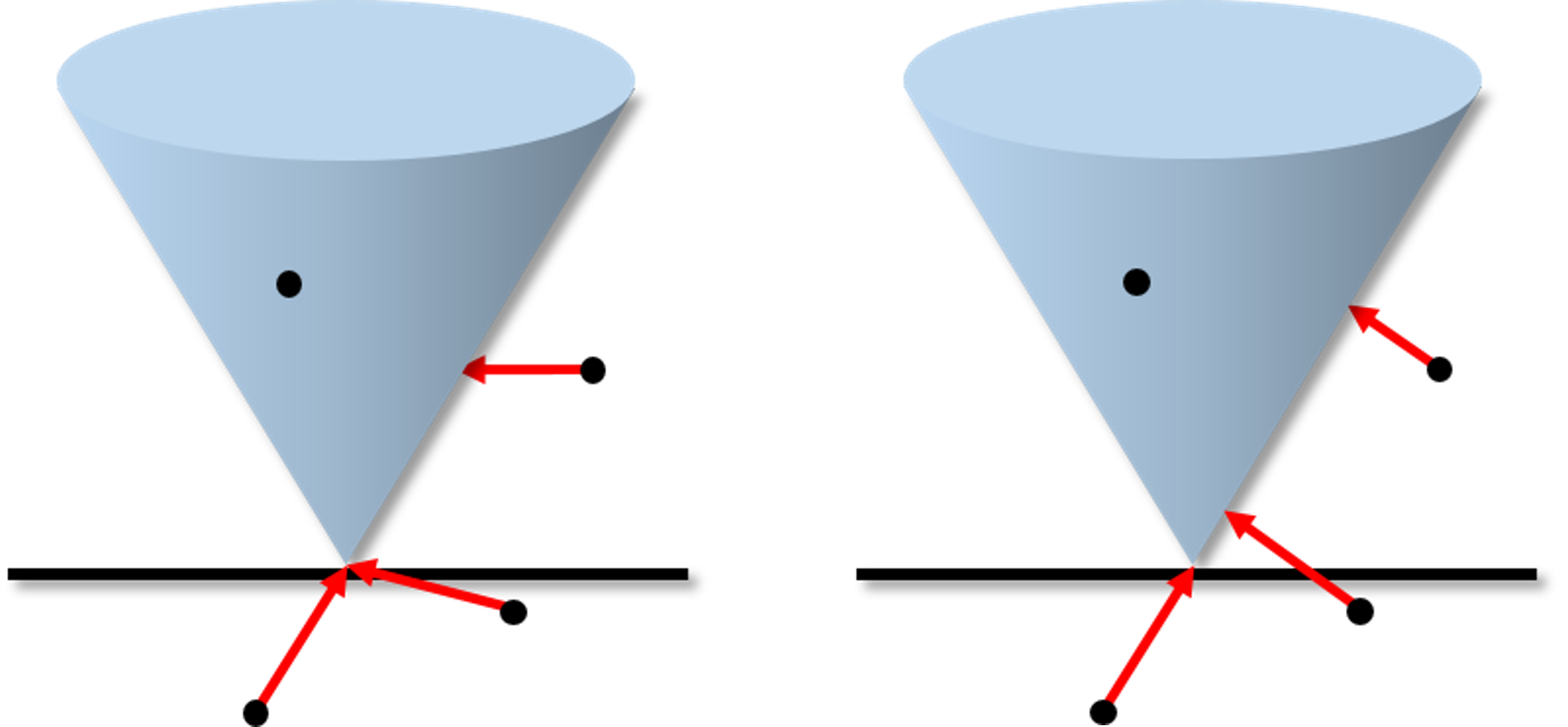}
\caption{Comparison of the strict operator (left) and the proximal operator (right) for the friction cone projection. Black dot: operator input; red arrow: projection direction.}
\label{fig:strict_proximal}
\end{figure}

Based on the closed-form operation \eqref{eq:closedform_operator}, solving \eqref{eq:prob_surrogate} can be now expressed as solving following nonlinear equation:
\begin{align} \label{eq:semismooth_eq}
\begin{split}
r(\hat{v})
&= A\hat{v} - b - \sum_i J_i^T \lambda_i \\
&= A\hat{v} - b - \sum_i J_i^T T(-\beta J_i\hat{v}-u_i-\beta e_i)
\end{split}
\end{align}
then computing $z=J\hat{v}+\frac{1}{\beta}(u+\lambda)$ accordingly.
Due to the projection operator \eqref{eq:strict_operator}, $r:\mathbb{R}^n\rightarrow\mathbb{R}^n$ is a continuous, yet semismooth equation.
Therefore, one can handle the surrogate problem by solving this nonlinear equation \eqref{eq:semismooth_eq} using the Newton method, whose theories developed under semismooth case \cite{hintermuller2010semismooth} by employing the generalized derivatives. 

However, typical (semismooth) Newton methods are known to exhibit superlinear convergence near the solution but lack robustness. 
Although the Prop.~\ref{prop:surrogate_feasibility} ensures the existence of solution and we attempt with various globalization techniques based on backtracking/edge-aware line-search and trust-region methods, we found that none provided sufficient robustness. 
This is critical considering that in physics simulation, as numerous iterations are required at each time step, and even a single failure can lead to significant consequences.
Furthermore, the derivative of the closed-form operator \eqref{eq:closedform_operator} might become non-symmetric in contact cases, and cannot guarantee that $\frac{dr}{d\hat{v}}$ will always be non-singular. This issue makes the computation both expensive and unreliable. 
Consequently, we have developed two variations of the augmented Lagrangian tailored for the multi-contact NCP form \eqref{eq:prob_original_slack}: the cascaded Newton-based augmented Lagrangian method (CANAL) and the subsystem-based ADMM (SubADMM), which are presented in the following sections.

\section{Cascaded Newton-Based Augmented Lagrangian} \label{sec:canal}

\subsection{Cascaded Structure}

A crucial issue of the Newton-based solution of \eqref{eq:semismooth_eq} is that the landscape of the merit function $\frac{1}{2}\| r(\hat{v})\|^2$ is non-convex.
Our core strategy to address this issue employs a cascaded method that relaxes each surrogate problem into a convex form, facilitating fast and stable solutions, while updating terms at each AL step to compensate for discrepancies between the convex problem and the original NCP.
For the convex relation, we utilize the equivalence of $(z,\lambda) \in \mathcal{S}_c$ and \eqref{eq:contactcon} with the following condition:
\begin{align} \label{eq:contactcon_equiv}
    \mathcal{C} \ni \lambda_i \perp z_i + \underbrace{\begin{bmatrix}
     0 \\ 0 \\ \mu_i \| z_{i,t} \|   
    \end{bmatrix}}_{p_i} \in \mathcal{C}^* 
\end{align}
where $\mathcal{C}^*$ denotes the dual cone of \(\mathcal{C}\). 
This equivalence can be easily verified, as we refer \cite{acary2011formulation} for details. 
The reformulated relation in \eqref{eq:contactcon_equiv} essentially constitutes a cone complementarity condition, if the perturbation term $p_i$ is excluded.

A key idea of our cascaded Newton approach is to substitute the perturbation term $p_i$ by borrowing $z_i$ from the previous AL iteration. 
In other words, we treat $p_i$ as a constant in every surrogate problem, and temporarily consider the relationship between $z_i$ and $\lambda_i$ as a cone complementarity condition. 
Consequently, in the $(l+1)$-th AL iteration, we solve the following nonlinear equation that replaces the strict operator \eqref{eq:closedform_contact} with the proximal operator:
\begin{align} \label{eq:closedform_contact_prox}
\begin{split}
    r(\hat{v}^{l+1}) &= A\hat{v}^{l+1}-b -\sum_i J_i^T\lambda_i^{l+1} \\
    \lambda_i^{l+1} &= \Pi_{\mathcal{C}}^\text{prox} (\underbrace{-\beta J_i\hat{v}^{l+1}-u_i^l-\beta\tilde{e}_i^l}_{\tilde{\lambda}_i^*}) 
\end{split}
\end{align}
where $\tilde{e}_i^l = e_i + p_i^l = e_i + \left[0 \; 0 \; \mu_i \| z_{i,t}^l \|\right]^T$.
Even after this replacement, the nonlinear equation in \eqref{eq:closedform_contact_prox} remains semismooth. However, we can demonstrate that it is integrable, as detailed in the following proposition. Note that to streamline the explanation, we will focus exclusively on the contact constraints below, as the other types (i.e., hard and soft) follow straightforwardly.
\begin{proposition}
    The function $r(\hat{v})$ from \eqref{eq:closedform_contact_prox} is the derivative of the following strongly-convex function:
    \begin{align} \label{eq:canal_convex_objective}
    h(\hat{v}) = \frac{1}{2} \hat{v}^T A \hat{v} - b^T \hat{v} + \sum_i \frac{1}{2\beta} \| \lambda_i \|^2
    \end{align}
\end{proposition}
\begin{proof}
    The derivative of $h(\hat{v})$ can be expressed as:
    \begin{align*}
        \frac{dh(\hat{v})}{d\hat{v}} 
        &= A \hat{v} - b - \sum_i J_i^T \frac{d\lambda_i}{d\lambda_i^*}^T \lambda_i \\
        &= A \hat{v} - b - \sum_i J_i^T \lambda_i
    \end{align*}
    The latter equality holds due to the identity $\lambda_i^T(\lambda_i - \lambda_i^*) = 0$ in the proximal operator. 
    The symmetric positive-definite property of $A$ ensures that the quadratic term is strongly convex.
    Furthermore, since the squared distance to a convex set is convex, $\| \lambda_i \|^2$ is convex with respect to $\lambda_i^*$, and thus also for $\hat{v}$.
    Therefore, $h(\hat{v})$ is a strongly-convex function.
\end{proof}

This result is closely related to those presented in \cite{todorov2014convex, castro2022unconstrained}, although the objective function is defined differently based on our AL-based formulation. 
Given this property, we can apply the exact Newton method to the strongly-convex function \eqref{eq:canal_convex_objective} by computing the derivative of $r(\hat{v})$ (i.e., the Hessian), which is proven to exhibit global convergence \cite{nocedal1999numerical}.

\subsection{Newton Step}

Computing the derivative of $r(\hat{v})$ in \eqref{eq:closedform_contact_prox} with respect to $\hat{v}$ is straightforward, except for the part involving $T$.  
As the operator $T$ is a proximal operator on a friction cone (see Fig.~\ref{fig:strict_proximal}), it involves a continuous concatenation of three formulaic forms, yet the function is semismooth at the connection points. 
Below, we provide derivative of each form which can be obtained from a few algebraic calculation:
\begin{align} \label{eq:canal_hessian_33}
    &\frac{d\lambda_i}{d\tilde{\lambda}_i^*} = 
    \begin{cases}
        0_{3\times 3}, & \mbox{open} \\
        I_{3\times 3}, & \mbox{stick} \\
        \frac{1}{\mu^2+1}
        \begin{bmatrix}
            \mu_i^2 I_{2\times 2} + \frac{\mu_i\tilde{\lambda}_{i,n}^*}{\|\tilde{\lambda}_{i,t}^*\|}P(\bar{\lambda}_{i,t})  & \mu_i\bar{\lambda}_{i,t}^T  \\
            \mu_i\bar{\lambda}_{i,t} & 1
        \end{bmatrix}, & \mbox{slip}
    \end{cases}
\end{align}
where $\bar{\lambda}_{i,t}^*$ is the normalized vector of $\tilde{\lambda}_{i,t}^*$ and $P(\bar{\lambda}_{i,t}) = I - \bar{\lambda}_{i,t}\bar{\lambda}_{i,t}^T$ is the tangential projection matrix.
Then the derivative can be written as
\begin{align} \label{eq:canal_hessian}
    \frac{dr(\hat{v})}{d\hat{v}} = A + \sum_i \beta J_i^T\frac{d\lambda_i}{d\tilde{\lambda}_i^*} J_i.
\end{align}
Due to the structure given in \eqref{eq:canal_hessian_33}, and consequently the matrix \eqref{eq:canal_hessian}, is guaranteed to be symmetric positive definite, therefore always invertible.
Followingly, the direction of the Newton step is computed as
\begin{align} \label{eq:canal_newtonstep}
    d(\hat{v}) = -\left(\frac{dr(\hat{v})}{d\hat{v}}\right)^{-1} r(\hat{v})    
\end{align}
where the $d(\hat{v})$ denotes the direction of $\hat{v}$ update.

Computation of the step \eqref{eq:canal_newtonstep} requires the linear solving of \eqref{eq:canal_hessian}, therefore assemble and factorization of the matrix is necessary.
For better efficiency, we can exploit sparsity pattern of the inertia matrix and the constraint Jacobian during the process.

\subsection{Exact Line-Search}

Drawing from well-known convex optimization theory \cite{nocedal1999numerical}, we can guarantee that \eqref{eq:canal_newtonstep} provides a descent direction. 
However, we still need to integrate a suitable line-search scheme to ensure global convergence. 
Here, the line-search problem can be described as following one-dimensional, strictly convex optimization problem:
\begin{align} \label{eq:canal_linesearch}
    \min_{\alpha > 0} f(\hat{v}+\alpha d(\hat{v})).
\end{align}   
Similar to \cite{castro2022unconstrained}, we can find a globally optimal solution of the problem \eqref{eq:canal_linesearch} using the \texttt{rtsafe} algorithm, which effectively combines the one-dimensional Newton-Raphson method and a bisection scheme.
In practice, we find that the Newton step, when combined with the aforementioned exact line-search, performs robustly even with large values of $\beta$. 
This approach significantly enhances the robustness of the simulator, as the standard semismooth Newton method on \eqref{eq:semismooth_eq} often leads to failures, especially for large $\beta$. The effectiveness is attributable to the integration of our cascaded scheme and the well-established theories in convex optimization literature.

\subsection{Warm-Start and Penalty Parameter Update} \label{subsec:canal_warmstart_penalty}

At each Newton loop, we can warm-start $\hat{v}$ from the value of the previous CANAL loop. This effectively reduces the number of necessary Newton steps in practice, as the optimal solution of the inner convex optimization should be similar as the AL iteration converges. Typically, with warm starting, we find that one or two Newton iterations often suffice after some progress has been made in the CANAL iteration.
Therefore, the computational cost per iteration step tends to decrease. 
For the first iteration, we can warm-start $\hat{v}$ and $u$ (and therefore also $z$) by using the values from the previous time step. 

The penalty parameter $\beta$ plays a crucial role in the CANAL algorithm. Typically, a high value of $\beta$ improves the convergence of the residual $\| J\hat{v} - z \|$ to zero. However, it also makes the convex problem numerically stiff, thus requiring additional Newton iterations to solve. 
Consequently, we begin with a moderate value of $\beta$ ($10^4$ in our cases) and increase it if the residual value is not sufficiently reduced. The update rule for increasing $\beta$ is defined as follows:
\begin{align} \label{eq:canal_penaltyupdate}
\beta \leftarrow \min(\kappa\beta, \beta^{\max}),
\end{align}
where $\kappa > 1$ is the hyperparameter. 
This rule includes restriction of $\beta$ from becoming unnecessarily large, thus bounding the stiffness of \eqref{eq:canal_penaltyupdate} to circumvent numerical instability.
Note that, if we perform only a single iteration on CANAL, it is equivalent to solving a soft convex formulation of contact as in \cite{mujoco,drake}. 
In this regard, CANAL can be considered as their extension, refining approximations and converging to near-rigid behavior through the update of primal-dual variables.
In practice, we find that it takes only a few iterations to achieve plausible behavior in simulations, and after several iterations, it tends to converge to very high accuracy.
The overall CANAL algorithm is summarized in Alg.~\ref{alg:canal}.

\begin{algorithm}[!t]
\SetAlgoLined
\caption{Multi-Contact Simulation via CANAL} 
\label{alg:canal}
\While{simulation}{
initialize $l=0,\hat{v}^0, z^0, \beta>0,\kappa>1,0<\eta<1$  \\
\While{CANAL loop}
{
initialize $\hat{v}^{l+1} \leftarrow \hat{v}^l$ \\
compute $\tilde{e}$ based on $z^l$ \\
\While{Newton loop}
{
compute $r(\hat{v}^{l+1})$ \eqref{eq:closedform_contact_prox} \\ 
\If{$\| r(\hat{v}^{l+1}) \| < \theta_{th}^N$}
{
\textbf{break}
} 
compute Newton step $d(\hat{v}^{l+1})$ \eqref{eq:canal_newtonstep} \\
compute $\alpha$ via exact line-search \eqref{eq:canal_linesearch} \\
$\hat{v}^{l+1} \leftarrow \hat{v}^{l+1} + \alpha d(\hat{v}^{l+1})$
}
update $z^{l+1}$ and multiplier $u^{l+1}$ \eqref{eq:AL_multiplier_update} \\
\eIf{$\|J\hat{v}^{l+1}-z^{l+1}\| < \theta_{th}^{AL}$} 
{
\textbf{break}
}
{
\If{$\|J\hat{v}^{l+1}-z^{l+1}\| > \zeta\|J\hat{v}^{l}-z^{l}\|$} 
{
update $\beta$ \eqref{eq:canal_penaltyupdate} \\
}
}
$l \leftarrow l+1$
}
update system state using $\hat{v}^{l+1}$
}
\end{algorithm}

\section{Subsystem-Based Alternating Direction Method of Multiplier}
\label{sec:subadmm}

While the CANAL-based multi-contact simulation described in Sec.~\ref{sec:canal} exhibits fast convergence and stable constraint handling in practice, its scalability may be limited by the need to compute at least one Newton step for each AL iteration. 
Although we fully exploit the sparsity pattern, the Hessian matrix may become fully dense in the worst-case scenarios (e.g., long kinematic chains, dense coupling), thereby significantly increasing the complexity of the factorization process.
Consequently, this approach can become computationally expensive when dealing with large-DOF multibody systems that include numerous objects.

One reasonable option in this regard is to adopt the methodology of ADMM, which, instead of solving the coupled problem of $(\hat{v}, z)$, performs alternating computation for each $\hat{v}$ and $z$.
By employing this alternating approach, the problem can be decomposed into a closed-form operator for $z$ from \eqref{eq:closedform_operator} and a linear problem for $\hat{v}$ from \eqref{eq:prob_surrogate}:
\begin{align} \label{eq:vanilla_admm_linear}
    (A+\beta J^TJ)\hat{v} = b + J^T(\beta z - u)
\end{align}
Although this vanilla ADMM allows matrix factorization to be performed only once for each time step, its computational efficiency diminishes with increasing system size, and also the sparsity pattern of the matrix in \eqref{eq:vanilla_admm_linear} is same with the Hessian matrices used in CANAL. 
Additionally, in practice, ADMM often necessitates tuning of the parameter $\beta$ during iterations to achieve optimal performance, which may require re-factorization of the matrix.
Hence, we introduce a novel algorithm termed subsystem-based ADMM (SubADMM), designed to offer enhanced scalability with parallelization capabilities.

\subsection{Subsystem-Based Reformulation} \label{subsec:subadmm_reformulation}

In our robotic simulation, we assume that the multibody system is composed of a kinematic chain \cite{featherstone2014rigid}, where each body is connected to its parent body via various types of joints (fixed, floating, revolute, prismatic, etc.). 
We then define the notion of a \textit{subsystem} as a single subtree rooted at the ground. For example, a single floating rigid body or a single robot (each of robotic arm, humanoid, etc.) is regarded as a subsystem.

To better leverage the subsystem structure, we adopt a variations on the definition of augmented Lagrangian compared to the one described in Sec.~\ref{sec:multicontactAL}.
We first rewrite the multi-contact simulation problem \eqref{eq:prob_original} as follows:
\begin{align} \label{eq:prob_subdyn}
\begin{split}
    &A_j\hat{v}_j = b_j + \sum_i J_{ij}^T\lambda_{i}~~\forall j\in\left\{1,\cdots,N\right\} \\
    &\text{s.t.}~~(\hat{V}_i,\lambda_i) \in \mathcal{S}_{c,i}~~\forall i\in\left\{1,\cdots,M\right\}
\end{split}
\end{align}
where $M$ is the number of constraint, $N$ is the number of subsystem, and $\hat{V}_i$ is the subset of $\left\{\hat{v}_1,\cdots,\hat{v}_N\right\}$ that contributes to the $i$-th constraint.   
Here, $J_{ij}$ are defined only for $(i,j)$ such that $\hat{v}_j\in\hat{V}_i$. 
Note that our reformulation does not rely on any assumptions about the system. In the unconstrained case, the dynamics of each subsystem are readily decoupled (thus $A_j\hat{v} = b_j~\forall j$). The coupling between subsystems is modeled by constraint forces as described in \eqref{eq:prob_subdyn}, which include both intra- and inter-subsystem interactions.

\begin{figure}[t]
\centering
\includegraphics[width=8cm]{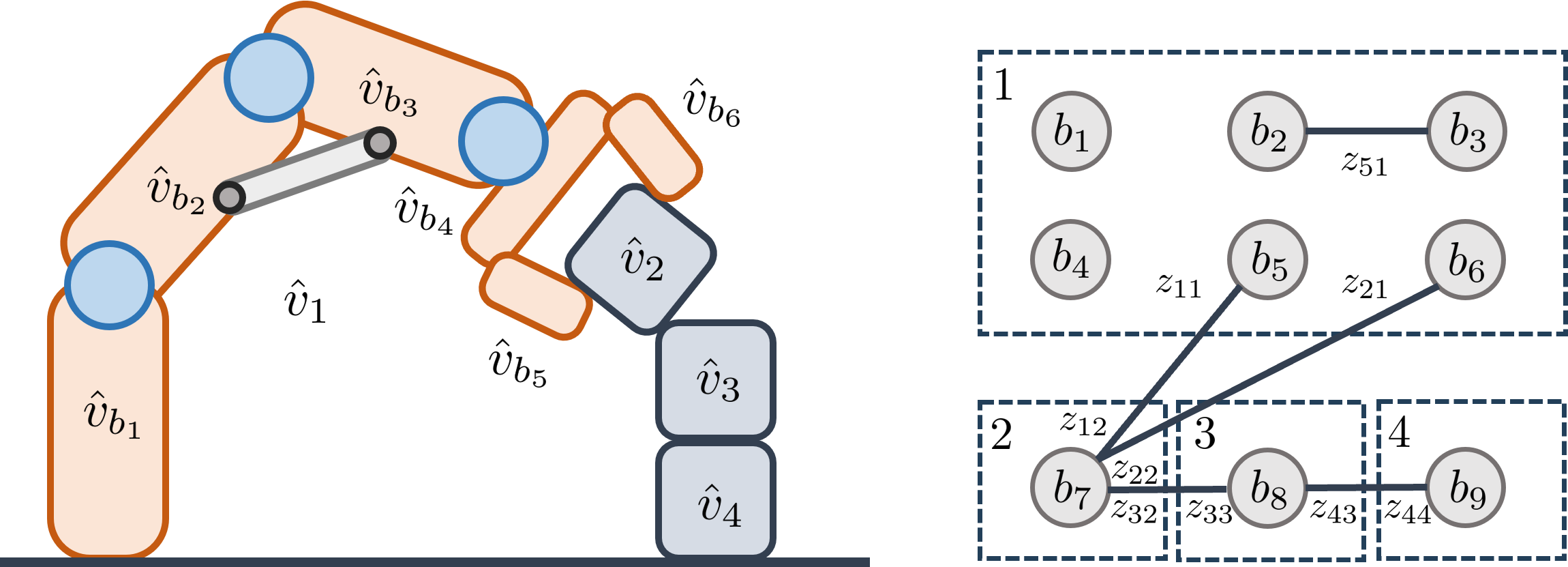}
\caption{Illustrative example demonstrating division and slack variable definition in SubADMM. Left: A multibody system with contacts comprising 4 subsystems, including 1 articulated body (robot) and 3 rigid bodies. Right: Corresponding graphical representation.}
\label{fig:subadmm_division}
\end{figure}

In robotic systems, constraints are applied either intra- or inter-bodies (e.g., contacts, tendons) or to joints (e.g., limits, controls).
Based on this insight, we also reformulate the structure of the Jacobian.
To illustrate the idea, let us consider the example in Fig.~\ref{fig:subadmm_division}.
In this example, the Jacobian for the first constraint (contact between the 5-th and 7-th bodies) can be written as follows:
\begin{align*}
    J_1 = 
    \begin{bmatrix}
        J_{1,b_5}J_{b_5,1} & J_{1,b_7}J_{b_7,2}
    \end{bmatrix}
\end{align*}
where $J_{b_\star,\star}$ maps the joint space to the body space, and $J_{\star,b_\star}$ maps the body space to the constraint space.
In this case, Jacobian for each subsystem is naturally defined as follows:
\begin{align} \label{eq:jacobian_body_splitting}
    J_{11} = J_{1,b_5}J_{b_5,1} \quad J_{12} = J_{1,b_7}J_{b_7,2}
\end{align}
Meanwhile, in the case of an inter-subsystem constraint, such as the 5-th constraint in Fig.~\ref{fig:subadmm_division}, acting internally on the first subsystem, we express the Jacobian as:
\begin{align} \label{eq:sub_jacobian}
    J_{51} = 
    \begin{bmatrix}
    J_{5,b_{2}}J_{b_{2},1} \\ J_{5,b_{3}}J_{b_{3},1}
    \end{bmatrix}
\end{align}
which splits the original Jacobian $J_{5,b_{2}}J_{b_{2},1}+J_{5,b_{3}}J_{b_{3},1}$ and stack it row-wise\footnote{For consistency, $\lambda_i$ in \eqref{eq:prob_subdyn} should technically be stacked row-wise for this case, but we maintain the current notation for simplicity.}.  
This reformulation is similarly applied for the constraints to joints.
Below, we describe how this reformulation can lead to an efficient ADMM process. 

\subsection{ADMM for the Reformulation}

Based on the reformulation described in Sec.~\ref{subsec:subadmm_reformulation}, we define slack variable to equivalently express the problem similar to \eqref{eq:prob_original_slack}:
\begin{align} \label{eq:prob_subdyn_slack}
\begin{split}
    &A_j\hat{v}_j = b_j + \sum_i J_{ij}^T\lambda_{i}~~\forall j\in\left\{1,\cdots,N\right\}\\
    &\text{s.t.}~~(Z_i,\lambda_i) \in \mathcal{S}_{c,i},~~z_{ij} = J_{ij}\hat{v}_j~~\forall i\in\left\{1,\cdots,M\right\}
\end{split}
\end{align}
where $Z_i$ is the set of slack variables $z_{ij}$ such that $(i,j)$ satisfies $\hat{v}_j\in\hat{V}_i$. 
Fig.~\ref{fig:subadmm_division} gives an example of how $z_{ij}$ are defined under the subsystem-based structure.
Then from the augmented Lagrangian structure described in Sec.~\ref{sec:multicontactAL}, surrogate problem for the reformulation \eqref{eq:prob_subdyn_slack} can be formulated as follows:
\begin{align} \label{eq:prob_sub_surrogate}
\begin{split}
(A_j + \sum_{i} \beta J_{ij}^TJ_{ij})\hat{v}_j - \beta \sum_{i} J_{ij}^Tz_{ij} &= b_j - \sum_{i} J_{ij}^Tu_{ij}~~\forall j\\
-\beta J_{ij}\hat{v}_j + \beta z_{ij} &=  u_{ij} + \lambda_i~~\forall i,j \\
\text{s.t.}~~&(Z_i,\lambda_i) \in \mathcal{S}_{c,i}~~\forall i.
\end{split}
\end{align}
The problem described in \eqref{eq:prob_sub_surrogate} still involves coupling between all $\hat{v}_j$ and $z_{ij}$. Therefore, solving it all at once would negate the benefits of using a subsystem-based representation.

However, by leveraging the ADMM structure \eqref{eq:ADMM}, an alternate resolution for each variable is performed, allowing us to capitalize on the subsystem-based partitioning described above.
By solving \eqref{eq:prob_sub_surrogate} with respect to $\hat{v}$ first, it reduces to following linear problem for all $j$:
\begin{align} \label{eq:subadmm_vhat}
(A_j + \sum_{i} \beta J_{ij}^TJ_{ij})\hat{v}_j^{l+1} = b_j + \sum_{i} J_{ij}^T(\beta z_{ij}^l-u_{ij}^l)
\end{align}
where $l$ denotes the iteration index.
The process described in \eqref{eq:subadmm_vhat} essentially involves obtaining a linear solution of size $\text{dim}(\hat{v}_j)$ for each subsystem, and always solvable from the positive definite property of the left-most matrix.
A crucial difference between \eqref{eq:vanilla_admm_linear} is that the linear problem is much smaller, while each of them can be solved in parallel. 
Each $A_j + \sum_{i} \beta J_{ij}^TJ_{ij}$ can be pre-factorized before iteration, as they remain invariant unless $\beta$ is modified.

Then, solving \eqref{eq:prob_sub_surrogate} with respect to $z_{ij}$ can be reduced to following problem for all $i$:
\begin{align} \label{eq:subadmm_z}
\begin{split}
    &\beta z_{ij}^{l+1} = \underbrace{\beta J_{ij}\hat{v}_j^{l+1} +u_{ij}^l}_{y_{ij}^{l+1}} + \lambda_i^{l+1}, ~~\text{s.t.}~~ (Z_i,\lambda_i) \in \mathcal{S}_{c,i}.
\end{split}
\end{align}
The solution of \eqref{eq:subadmm_z} can be performed independently for each $i$, allowing for parallelization across all constraints. 
Additionally, as described in Sec.~\ref{subsec:closed_form_slack}, the relation between $z$ and $\lambda$ is matrix-free, thus allowing for a closed-form solution process. However, the formulation needs to be slightly adjusted since multiple slack variables in $Z_i$ are involved in each constraint. 
For the case of hard constraint, \eqref{eq:closedform_hard} can be adjusted for the subsystem-based reformulation as follows:
\begin{align} \label{eq:closedform_hard_sub}
    &\lambda_i^{l+1} = \Pi_{\ge 0}\left (-\frac{\sum_{j} y_{ij}^{l+1} + \beta e_i}{\vert Z_i \vert}  \right)
\end{align}
where $\vert Z_i \vert$ denotes the cardinality of $Z_i$. 
Meanwhile, for the soft constraints, \eqref{eq:closedform_soft} can be adjusted as follows:
\begin{align} \label{eq:closedform_soft_sub}
    &\lambda_i^{l+1} = -\frac{b_i\sum_{j} y_{ij}^{l+1} + \beta k_ie_i}{b_i\vert Z_i\vert + \beta}.
\end{align}
Finally, for the contact constraints, we can use
\begin{align} \label{eq:closedform_contact_sub}
    \lambda_i^{l+1} = \Pi_{\mathcal{C}}^\text{strict} \left (-\frac{\sum_{j} y_{ij}^{l+1} + \beta e_i}{\vert Z_i \vert}  \right)
\end{align}
instead of \eqref{eq:closedform_contact}.
The resulting equations \eqref{eq:closedform_hard_sub}, \eqref{eq:closedform_soft_sub}, and \eqref{eq:closedform_contact_sub} still consist of simple algebraic operations, making them easy to compute.
Note that we directly employ the strict operator in \eqref{eq:closedform_contact_sub}, without adopting the cascaded approach with the proximal operator as used in CANAL. 
This is because ADMM strategically employs a single alternating solution without fully solving the surrogate problem, which provides conservative updates and maintains stability without requiring a specific convexification and globalization process.

After both alternating steps, the Lagrange multiplier is updated as follows:
\begin{align} \label{eq:subadmm_u}
\begin{split}
    u_{ij}^{l+1} 
    &= u_{ij}^l + \beta (J_{ij}\hat{v}_j^{l+1} - z_{ij}^{l+1}) \\
    &=-\lambda_i^{l+1}.
\end{split}
\end{align}
Hence, there is no necessity to store $u_{ij}$ separately; only $\lambda_i$ needs to be retained.
In summary, our SubADMM iterates between \eqref{eq:subadmm_vhat}, \eqref{eq:subadmm_z}, and \eqref{eq:subadmm_u}, with each step being naturally parallelizable per subsystem or per constraint as illustrated in Fig.~\ref{fig:subadmm_procedure}. This property ensures the scalability of the algorithm with respect to the number of subsystems and constraints. 

\begin{figure}[t]
\centering
\includegraphics[width=8cm]{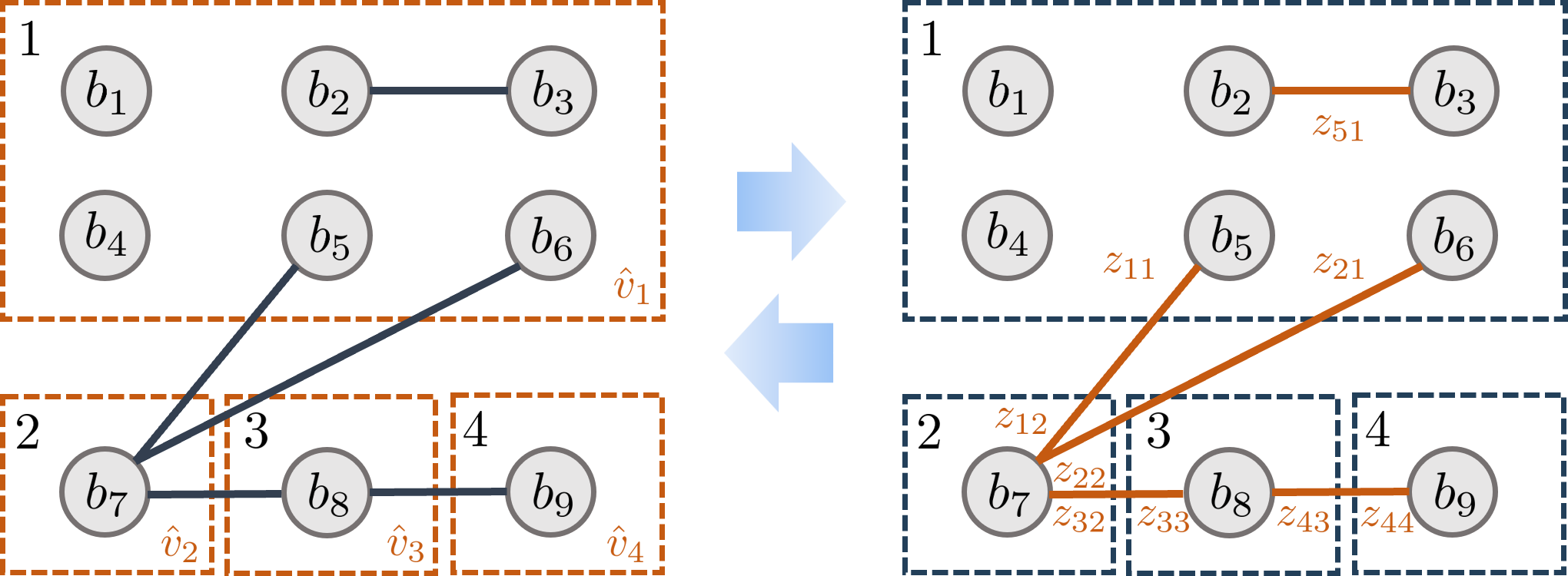}
\caption{Iteration structure of SubADMM. $\hat{v}$ is updated independently for each subsystem block, while $z$ is independently updated for each constraint factor.}
\label{fig:subadmm_procedure}
\end{figure}

\subsection{Factorization of Subsystem Matrices} \label{subsec:submatrix_factor}

In the SubADMM process described above, the size of each linear problem in \eqref{eq:subadmm_vhat} is determined by $\text{dim}(\hat{v}_j)$, which equals 6 for a subsystem consisting of a single rigid body. 
However for subsystems with articulated body structures, it corresponds to the total DOF of joints.
Generally, the factorization of a matrix has a complexity of $\mathcal{O}(n^3)$. Consequently, there may be concerns that our algorithm could become susceptible to an increase in the degrees of freedom of the subsystem, such as in a tree with extensive length, unless $A_j + \sum_{i} \beta J_{ij}^T J_{ij}$ possesses a special structure.

However, the reformulation technique \eqref{eq:sub_jacobian} we established earlier for the Jacobian matrix, becomes crucial in addressing this concern.
To illustrate, let us derive the equation as follows:
\begin{align} \label{eq:submatrix_structure}
\begin{split}
A_j + \sum_{i} \beta J_{ij}^TJ_{ij} 
&= A_j + \sum_{i} \sum_k \beta J_{b_k,j}^TJ_{i,b_k}^TJ_{i,b_k}J_{b_k,j}  \\
&= J_{\mathcal{B}_j,j}^T 
\underbrace{\begin{bmatrix}
H_{b_{j,1}} & & \\
& \ddots & \\
& & H_{b_{j,\vert \mathcal{B}_j \vert }}
\end{bmatrix}}_{H_{\mathcal{B}_j}}
J_{\mathcal{B}_j,j} 
\end{split}
\end{align}
where $\mathcal{B}_j$ represents the set containing every body index in $j$-th subsystem, $J_{\mathcal{B}_j,j}$ denotes the Jacobian mapping from subsystem joints to all child body spaces, and $H_{b_k}\in\mathbb{R}^{6\times 6}$ can be interpreted as the virtual effective matrix defined for each body, expressed as
\begin{align} \label{eq:submatrix_body}
H_{b_k} = M_{b_k} + \sum_i \beta J_{i,b_k}^TJ_{i,b_k}
\end{align}
where $M_{b_k}$ is the inertia matrix for the body originated from $A_j$.
One significant advantage of the structure \eqref{eq:submatrix_structure} is that $H_{\mathcal{B}_j}$ forms a block-diagonal matrix, ensuring that the entire matrix always maintains the same sparsity structure as the inertia matrix of the articulated body.
Hence, we can implement an efficient construction and factorization algorithm based on the kinematic-tree structure, which is well-known as Featherstone's algorithm \cite{featherstone2014rigid}. 
Specifically, we use composite rigid body algorithm \cite{featherstone2014rigid} that capitalizes on branch-induced sparsity to streamline computations, recursively navigating to the parent node to perform efficient fill-ins.
In Appendix~\ref{appendix:crba}, we details how the algorithm on \eqref{eq:submatrix_body} can be efficiently performed.
It is worth noting that, without the Jacobian reformulation described in \eqref{eq:sub_jacobian}, $H_{\mathcal{B}_j}$ does not exhibit block diagonal characteristics. As a result, the sparsity pattern becomes more complex and cannot be determined solely by the kinematic structure, but rather changes with each time step.

\begin{remark}
Based on the aforementioned matrix structure, rather than constructing and factorizing \eqref{eq:submatrix_structure}, we can solve the linear equation \eqref{eq:subadmm_vhat} using the articulated body algorithm \cite{featherstone2014rigid} at each SubADMM iteration step. 
While this approach strictly guarantees $\mathcal{O}(n)$ complexity for the body count, with no additional overhead for changes in $\beta$, we have observed that using factorization is often more efficient in practice, given that SubADMM typically requires a few tens of iteration. Nonetheless, this alternative remains as a viable option.
\end{remark}

\subsection{Convergence and Adaptive Penalty Parameter}

\begin{algorithm} [t]
\SetAlgoLined
\caption{Simulation using SubADMM} 
\label{alg:subadmm}
subsystem-based reformulation (Sec.~\ref{subsec:subadmm_reformulation}) \\
\While{simulation}{
initialize $\beta$ \\
$\forall j$ construct $A_j,b_j$ in parallel \\
$\forall i$ construct $e_i,J_i$ in parallel \\
$\forall j$ factorize $A_j + \sum_{i} \beta J_{ij}^TJ_{ij}$ in parallel (Sec.~\ref{subsec:submatrix_factor}) \\
initialize $l=0, z^0, u^0$ \\
\While{SubADMM loop}{
$\forall j$ update $\hat{v}_{j}^{l+1}$ in parallel \eqref{eq:subadmm_vhat}  \\
$\forall j$ update $Z_i^{l+1}$ in parallel \eqref{eq:subadmm_z} \\
$\forall i$ store $u_i^{l+1}$ \eqref{eq:subadmm_u} \\
compute residual $\theta_p,\theta_d$ \eqref{eq:subadmm_residual} \\
\eIf{$\theta_p+\theta_d < \theta_{th}$ or $l=l_{max}$} 
{
\textbf{break}
}
{
update $\beta$ \eqref{eq:subadmm_update_beta} (optional) \\
$\forall j$ refactorize $A_j + \sum_{i} \beta J_{ij}^TJ_{ij}$ in parallel 
}
$l\leftarrow l+1$ \\ 
}
update each subsystem state using $\hat{v}_j^{l+1}$
}
\end{algorithm}

For strictly convex optimization, ADMM is known to guarantee convergence at a linear rate \cite{boyd2011distributed}. 
Although our formulation shares similarities with the convex optimization \eqref{eq:prob_original_opt}, the multi-contact condition is generally not integrable, making theoretical convergence not well established in general. However, we empirically find that the convergence properties in our SubADMM-based simulations are comparable to those observed in convex optimization.

Typically, residuals in ADMM are defined in two kinds: primal and dual. 
In SubADMM, these definitions similarly hold as follows:
\begin{align} \label{eq:subadmm_residual}
\begin{split}
\theta_p &= \max_{i,j} \| J_{ij}\hat{v}_j - z_{ij} \| \\
\theta_d &= \max_j \| A_j\hat{v}_j - b_j -\sum_i J_{ij}^T\lambda_i \|
\end{split}
\end{align}
where $\theta_p$ and $\theta_d$ represent the primal and dual residuals, respectively. Here, the primal residual can be interpreted as the satisfaction of constraints, while the dual residual reflects the satisfaction of the dynamics equations.
Although we observe stable convergence of the residuals \eqref{eq:subadmm_residual} in SubADMM, the method still inherits well-known drawbacks associated with ADMM-style iterations. Specifically, the performance of the algorithm is heavily dependent on the strategy used to select the penalty parameter $\beta$.
A popular strategy is to adaptively tune $\beta$ based on the residual. Generally, a large $\beta$ reduces the primal residual, while a small $\beta$ reduces the dual residual.
Therefore, we can adopt the strategy of adjusting $\beta$ based on the following rule:
\begin{align} \label{eq:subadmm_update_beta}
    \beta = \beta \sqrt{\dfrac{\theta_p}{\theta_d}} \quad \text{if} \quad \theta_p>\gamma\theta_d~\text{or}~\theta_d>\gamma\theta_p
\end{align}
where $\gamma > 1$ is a hyperparameter.
This adjustment should accompany the refactorization of $A_j + \sum_{i} \beta J_{ij}^TJ_{ij}$. 
However, thanks to our subsystem-based division structure, this refactorization can also be performed in a parallelized and scalable manner, allowing us to reduce overhead and enable more frequent feedback adjustments.

Moreover, another crucial aspect we observe to address this issue is that the presence of large number of inactive constraints (the open case for contact) can impede convergence. 
Therefore, conducting thorough broad-phase collision tests to cull reasonable contact pairs is a significant step in enhancing convergence speed.
For the initialization of $\beta$, we consider the structure of the terms in \eqref{eq:submatrix_structure}. A practical strategy involves balancing the weighting between the dynamics-related term $A_j$ and the constraint-related term $\sum_i J_{ij}^T J_{ij}$, as suggested in general theoretical analysis \cite{ghadimi2014optimal,giselsson2016linear}.
From this perspective, $\beta$ in \eqref{eq:submatrix_body} can be interpreted as reflecting a pseudo-density (i.e., body mass divided by the number of contact points on it), given that the Jacobian $J_{i,b_k}$ maps the motion of the body frame to the motion of a specific point on the body (see also derivation in Appendix~\ref{appendix:crba}).
Based on this insight, we use the geometric mean of the pseudo-densities of all bodies as the initial value for $\beta$.

\subsection{Comparison with Previous Work}

Portions of the algorithm outlined in this section were presented in our previous work \cite{lee2023modular}, applying the idea of leveraging ADMM in physics simulations and the division of the entire multibody system into smaller parts. 
However, the scope of this article extends to more general theories and diverse variations of the augmented Lagrangian for robotic multi-contact simulations, as outlined in Sec.~\ref{sec:multicontactAL} and \ref{sec:canal}. 
For this section specifically, this article provides a complete subsystem division and parameter adaptation rule for subsystem-based ADMM. 
More importantly, we introduce novel reformulation techniques \eqref{eq:sub_jacobian}, enabling efficient factorization of submatrices described in Sec.~\ref{subsec:submatrix_factor}. This improvement enhances scalability, not only for the number of subsystems but also for the DOF of each subsystem.
Finally, a variety of new manipulation examples are implemented in simulations for the experiments.

\section{Examples and Evaluations} \label{sec:exampleeval}

In this section, we present several implementation examples to illustrate the advantages of the proposed solver algorithms. The key question we address here is whether our AL-based multi-contact solver (CANAL and SubADMM) can address the limitations of existing per-contact schemes in solving NCP posed in various robotic scenarios. 
As a universal baseline, we have implemented the Projected Gauss-Seidel (PGS) solver, which is widely utilized in numerous software applications (see Table~\ref{table:simulator_compare}) and can handle NCP without the need for model relaxation. 
Note that we do not consider methods that depend on specific model relaxations, such as direct pivoting schemes based on LCP formulation. 

Quantifying the accuracy of different multi-contact solvers is non-trivial. 
For each time step, it is essential to assess how well the solution $(\hat{v}, \lambda)$ satisfies two conditions: dynamics and contact constraints. 
For consistent comparison, we first project $\hat{v}$ using the relationship $\hat{v} = A^{-1}(b + J^T\lambda)$ based on the $\lambda$ result of solver, making the dynamics residual become zero. 
Subsequently, we compute $z$ using the following equation:
\begin{align*}
    z_i = J_i\hat{v} - \lambda_i + T(-J_i\hat{v} + \lambda_i - e_i)
\end{align*}
which is derived from the process described in Sec.~\ref{subsec:closed_form_slack}, and we measure the contact residual as $\|J\hat{v} - z\|$ divided by the number of contacts. This residual becomes zero only when the contact conditions are exactly satisfied. Using this strategy, we can fairly compare the accuracy of different solvers: dual-based (PGS) and primal-dual based (CANAL and SubADMM).
For all examples, we utilize a time step of $1/240~\rm{s}$.

For the code implementation, we utilize the C++ language, employing the Eigen matrix library \cite{eigenweb} for linear algebra operations and the OpenGL library for rendering. Computation time is measured on an Intel Core i5-13600KF CPU at 3.50 GHz.

\subsection{Bolt-Nut Assembly} \label{subsec:exp_boltnut}

\begin{figure}[t]
\centering
\includegraphics[width=8.4cm]{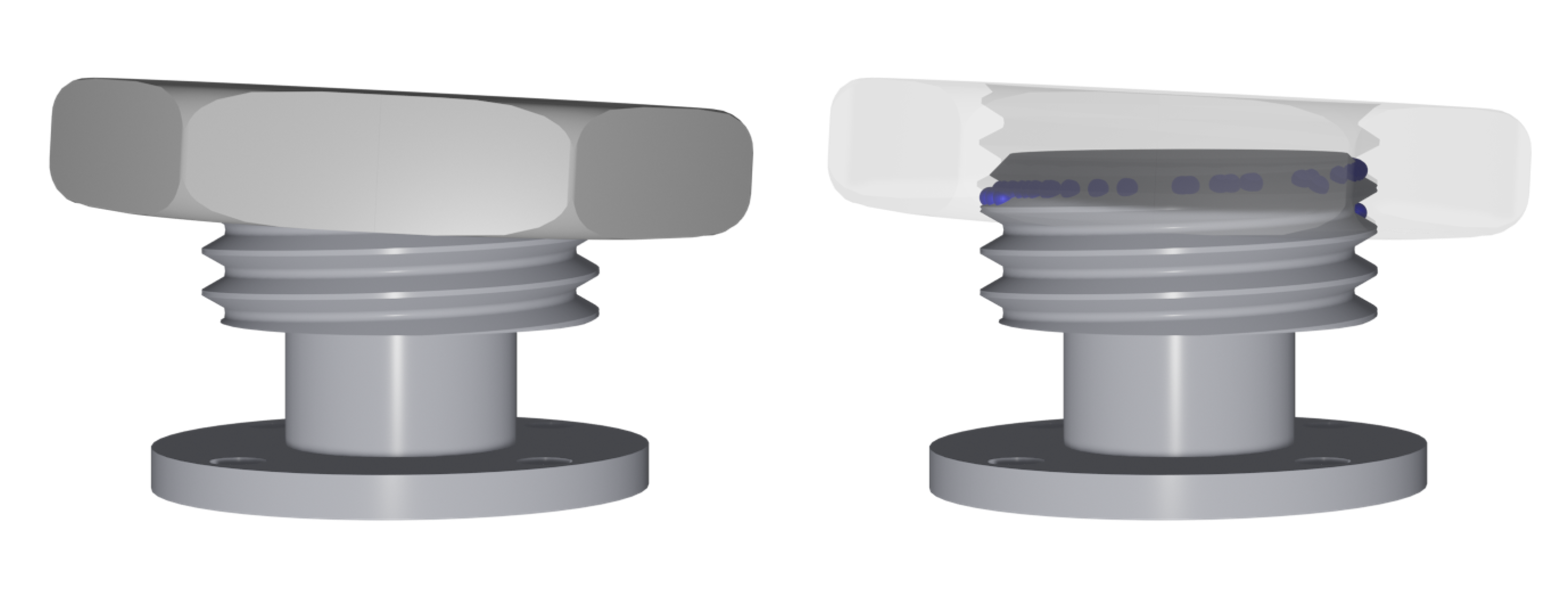}
\caption{Left: Visualization of an M48 bolt and nut used in our bolt-nut assembly simulation test. Right: Contact points visualized as blue spheres in the test configuration.}
\label{fig:boltnut_m48}
\end{figure}

\begin{figure}[t]
\centering
\includegraphics[width=8.4cm]{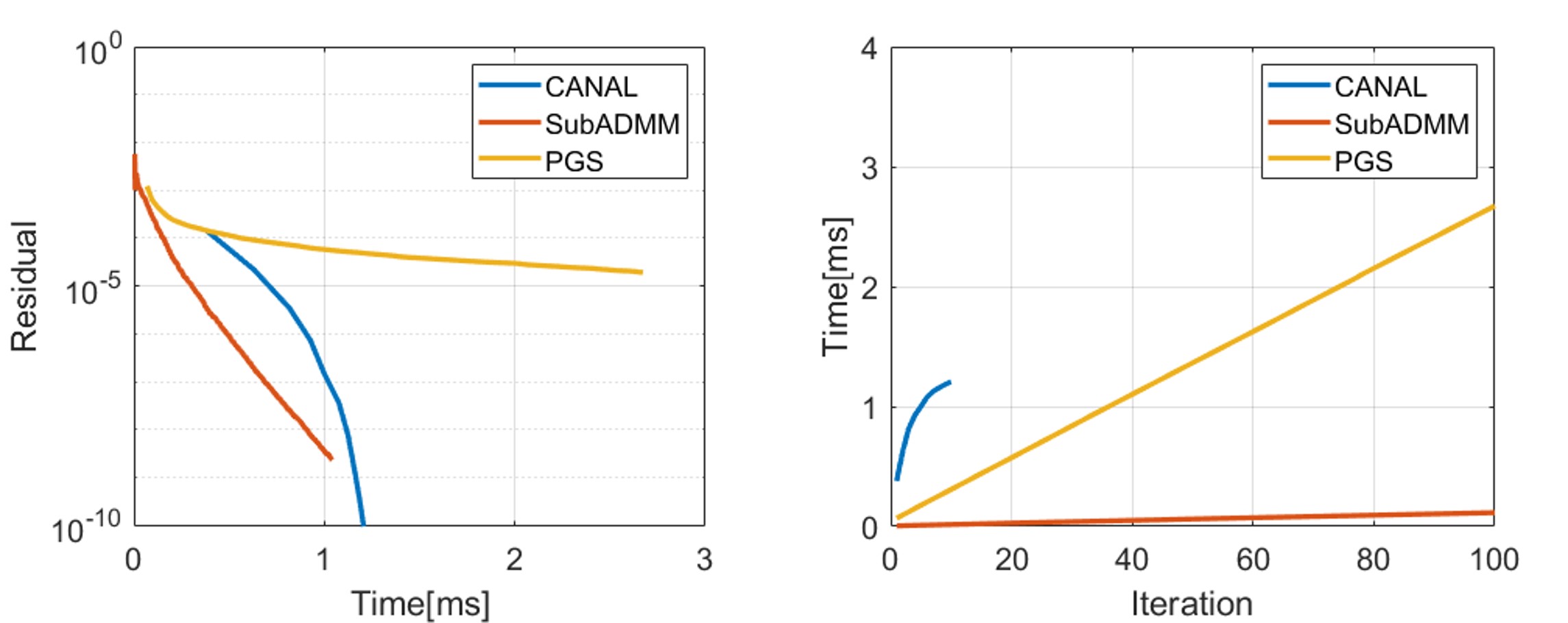}
\caption{Comparison of CANAL, SubADMM, and PGS for the bolt-nut assembly task simulation. Left: Residual decrease over computation time. Right: Computation time over iteration.}
\label{fig:boltnut_results}
\end{figure}

The first example we consider is the simulation of a bolt-nut assembly. This scenario is characterized by the intensive formation of contacts and stiff interactions due to the complexity of the geometry. As a result, the DOF for the constraints (typically hundreds) far exceed those of the system itself. 
In this context, two significant issues arise with the per-contact solver: 1) the Delassus operator becomes large and dense (i.e., many contacts are coupled), which slows down the iteration scheme, and 2) the intensive contact results in a limited feasible region or even leads to infeasibility, thereby slowing convergence.

The scenario is configured via a pair of M48 bolt and nut, visualized in Fig.~\ref{fig:boltnut_m48}. 
Here, collision detection process between the bolt and nut might be challenging. Therefore, we adopt a neural network-based signed distance function model presented in \cite{yoon2022fast}. Based on this model, we can precisely represent the surface of the bolt, while the nut is represented through multiple triangulated face to perform collision detection using Frank-Wolfe algorithm \cite{macklin2020local}. Additionally, if the number of detected contact points exceeds 120, which is empirically found to be impractical, we perform contact clustering \cite{jain1999data,yoon2022fast} to effectively reduce them.

\subsubsection{Single Step Test}

To precisely evaluate the quantitative performance, we first measure the results of running different solvers single step at the same state and inputs.
For test case generation, we sample $10$ configurations of a nut that form numerous contacts with the bolt (see Fig.~\ref{fig:boltnut_m48} for an example). We then apply $10$ random external wrenches to the nut, generating a total of $100$ test cases.
The performance of the solvers in such scenarios is depicted Fig.~\ref{fig:boltnut_results}. 
As demonstrated, CANAL and SubADMM solve the problems with higher accuracy compared to the PGS algorithm. Both algorithms can achieve residuals of $10^{-8}$ or less, whereas PGS struggles to converge quickly past $10^{-4}$. 
In particular, CANAL exhibits superlinear convergence, achieving complete convergence in about 10 iterations. While SubADMM quickly converges to a certain accuracy, due to its first-order nature, it shows lower accuracy than CANAL after a certain period. 
For the computation time per iteration, SubADMM requires significantly less compared to the other two algorithms as it requires negligible preparation phase cost and each iteration can be performed very quickly. CANAL takes more time per iteration compared to PGS, however, its rapid convergence leads to fewer iterations overall, resulting in shorter total computation time. 
While SubADMM and PGS incur uniform costs for each iteration, leading to a linear increase in computation time, computation time for CANAL grows sublinearly as iterations progress.
This is due to the effect of warm-starting explained in Sec.~\ref{subsec:canal_warmstart_penalty}, as the outer AL iteration approaches convergence, the number of required Newton iterations to solve the inner convex problem decreases.   
Overall, the results suggest that if we aim to achieve reasonable accuracy in a very short amount of time, SubADMM is a good option. Conversely, if very high accuracy is desired and more computational budget is available, CANAL is the preferable option.

\subsubsection{Assembly using Robot Manipulator}

We also perform a full assembly task simulation using a robotic manipulator. The manipulator comprises a Franka Panda arm equipped with a Hebi X5 gripper, with the nut attached to the gripper. We control the robot using joint-level impedance control to follow the desired assembly trajectory, and the snapshots are depicted in Fig.~\ref{fig:thumbnail}. For performance validation, we limit the computation budget for the solver to $0.5~\rm{ms}$ for each time step. As a result, simulations using CANAL and SubADMM successfully complete the assembly task. In contrast, PGS fails, as significant penetrations are generated due to its lack of convergence.

\subsection{Dish Piling}

The next scenario we implement involves piling dishes, a common situation in household environments. To compose the environment, we generate various types of dishes, including bowls, plates, and pots, as depicted in Fig.~\ref{fig:dishpile_objects}. As the shapes of the dishes are concave, this scenario is characterized by a large number of contacts distributed across the surfaces, leading to issues similar to those described in  Sec.~\ref{subsec:exp_boltnut}. Moreover, we model light bowls and plates (0.1 kg) beneath a heavy pot (5 kg), resulting in a challenging mass ratio for the stable simulation.

We develop a specially designed class of signed distance functions to represent the geometry of dishes, which is detailed in Appendix~\ref{appendix:dish_sdf}. This approach enables us to compute the signed distance using simple algebraic operations, such as rounding and revolution. Subsequently, we generate the corresponding triangulated faces to perform collision detection between dishes, employing the same Frank-Wolfe algorithm used in the bolt-nut assembly scenario.

\subsubsection{Single Step Test}

\begin{figure}[t]
\centering
\includegraphics[width=8.4cm]{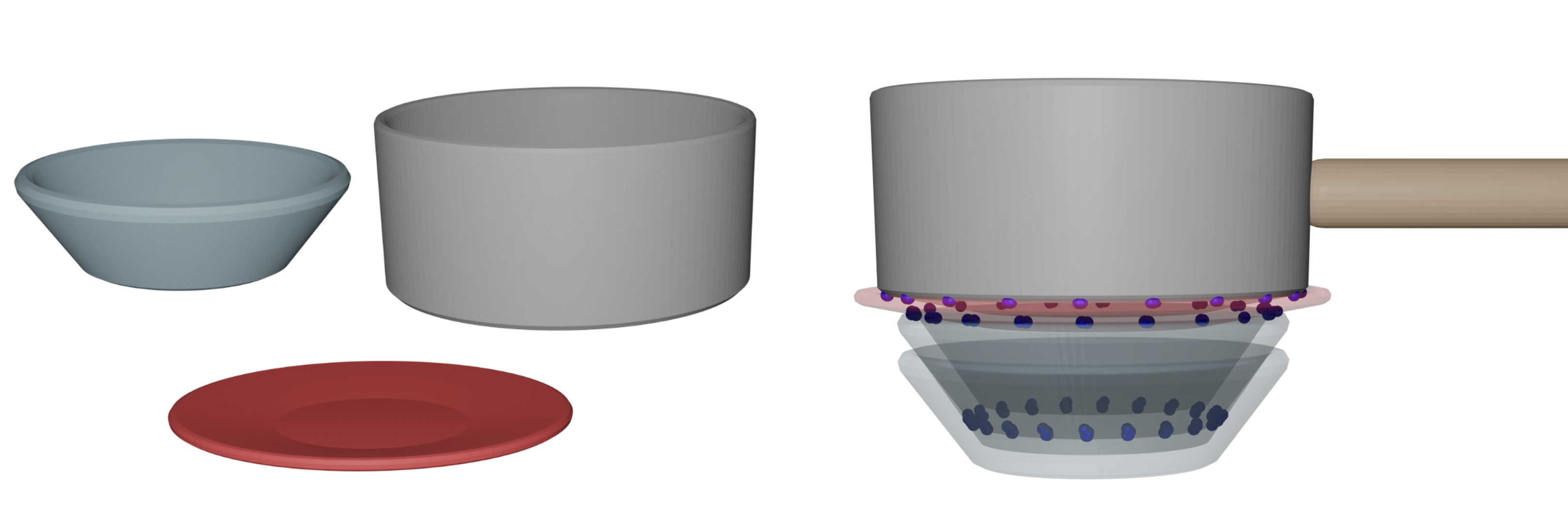}
\caption{Left: Visualization of a bowl, plate and pot used in our dish piling simulation test. Right: Contact points visualized as blue spheres in the test configuration.}
\label{fig:dishpile_objects}
\end{figure}

\begin{figure}[t]
\centering
\includegraphics[width=8.4cm]{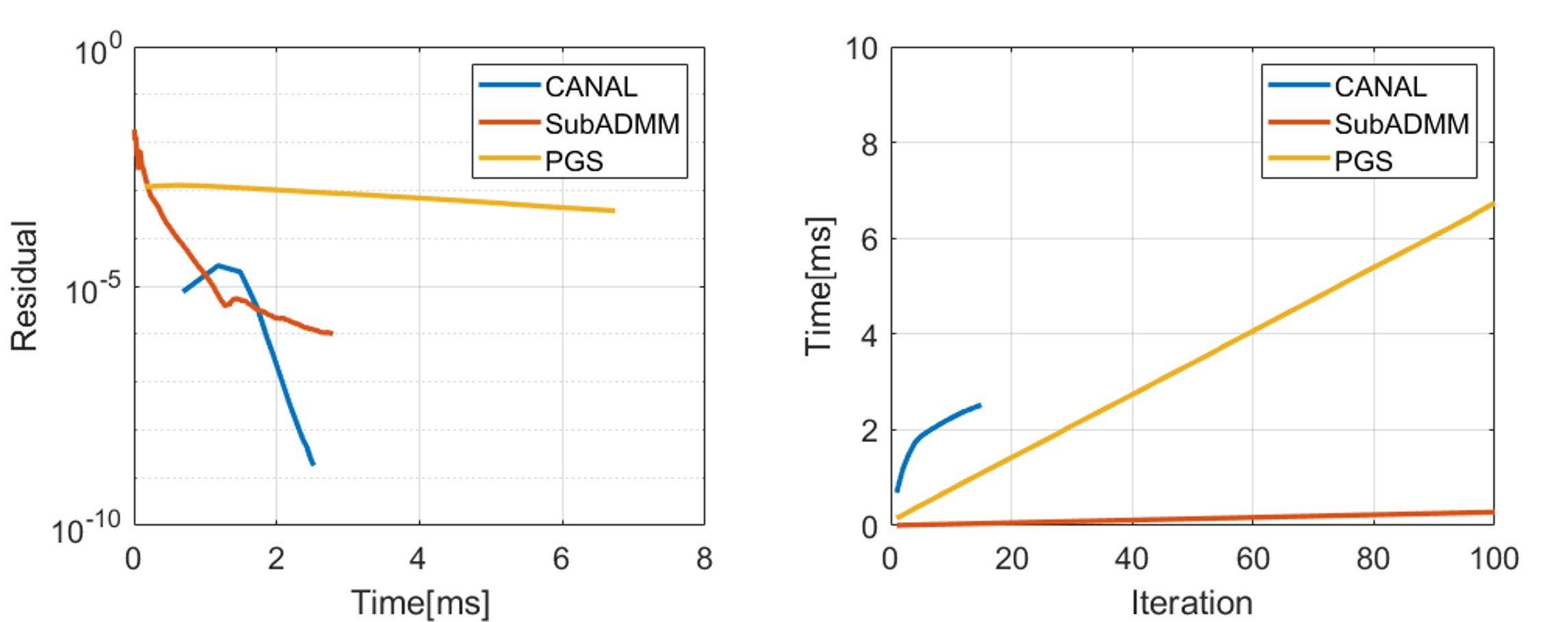}
\caption{Comparison of CANAL, SubADMM, and PGS for the dish piling simulation. Left: Residual decrease over computation time. Right: Computation time over iteration.}
\label{fig:dishpile_singlestep}
\end{figure}

For a single-step test, we first obtain the stacked pose of four dishes (bowl-bowl-plate-pot), as depicted in Fig.~\ref{fig:dishpile_objects}. 
We then apply random external wrenches to each dish, generating a total of 100 test cases. 
The performance of the solvers in these cases is depicted in Fig.~\ref{fig:boltnut_results}. 
As shown, CANAL and SubADMM achieve significantly better accuracy in less time compared to PGS. 
CANAL achieves the highest accuracy, with residuals under $10^{-8}$, and exhibits over linear convergence. SubADMM, demonstrating first-order convergence, struggles to achieve residuals under $10^{-6}$. Due to the odd mass ratio present in the environment, the differences in achievable residuals between the solvers are larger compared to those in the bolt-nut assembly test. This suggests that CANAL may be the more preferable option in this case, although SubADMM remains a viable choice for achieving moderate results in a very short time.
The trend in computation time per iteration is similar to that observed in bolt-nut assembly scenarios; per-iteration cost ranks as follows: CANAL $>$ PGS $>$ SubADMM, and the cost for each iteration in CANAL tend to decreases as the iterations proceed.

\subsubsection{Piling using Robot Manipulator}

We also perform a piling task simulation using a robotic manipulator composed of a Franka Panda arm equipped with an Allegro hand, bringing the total system dimension to 47. We employ a joint-level impedance controller to enable the robot to follow the desired grasp-and-place trajectory, as depicted in the snapshots in Fig.~\ref{fig:thumbnail}. Similar to the bolt-nut assembly simulation test, we limit the computation budget to 1~$\rm{ms}$ to compare the performance of different solvers. In this test, SubADMM and CANAL can successfully simulate the piling, while PGS fails, generating jittery movements due to a lack of convergence.

\subsection{Pouring}

\begin{figure}[t]
\centering
\includegraphics[width=8.4cm]{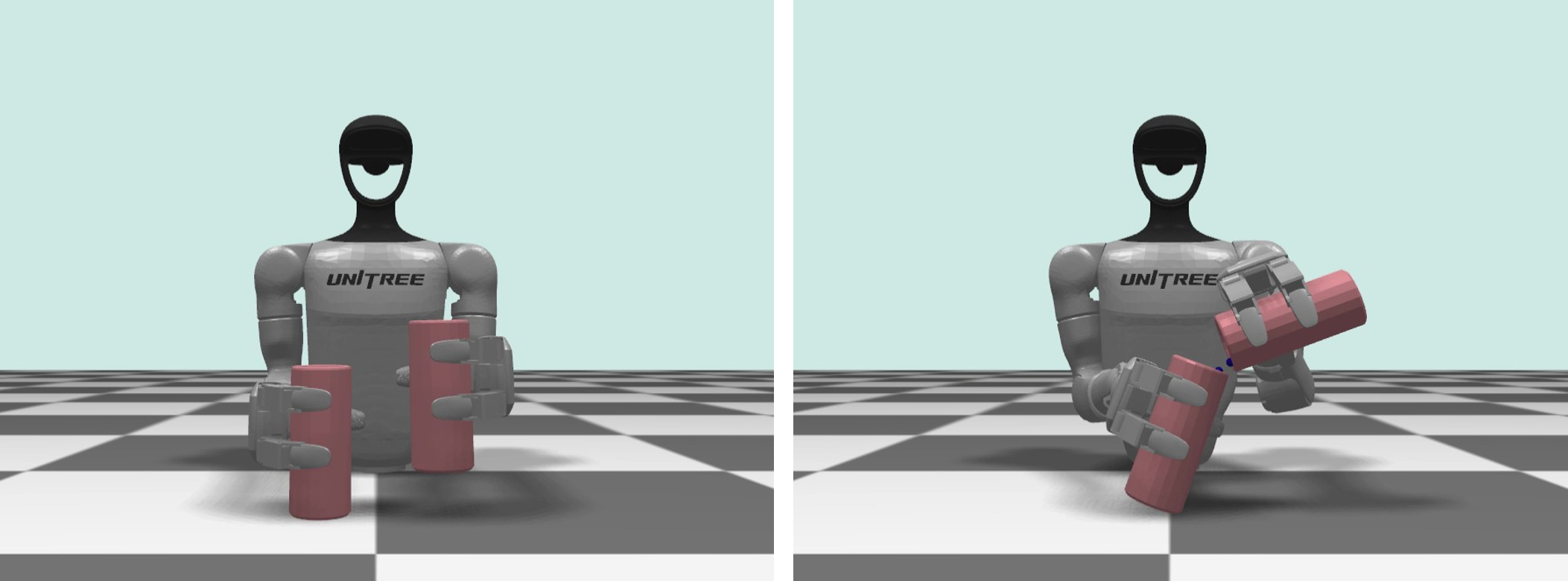}
\caption{Snapshots of a particle pouring task simulation. SubADMM excels in computation speed and scalability; CANAL in accuracy.}
\label{fig:pouring_snapshot}
\end{figure}

\begin{figure}[t]
\centering
\includegraphics[width=8.4cm]{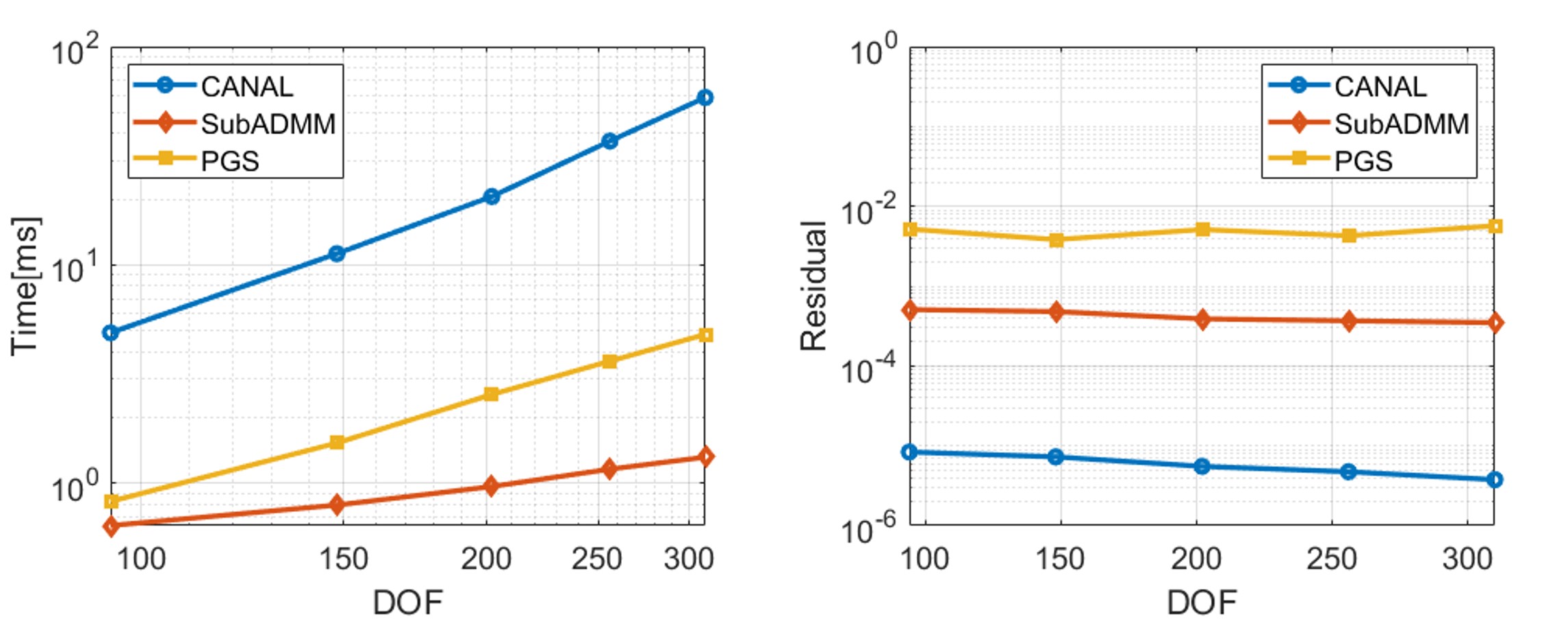}
\caption{Comparison of CANAL, SubADMM, and PGS for the pouring simulation. Left: Average computation time over system DOF. Right: Average residual over system DOF.}
\label{fig:pouring_scalability}
\end{figure}

We then simulate the pouring of particles contained in a bottle, employing a 28 DOF dual-arm manipulator with the upper body of a Unitree G1 with gripper. 
Also the particles are modeled as spheres with a radius of 5~$\rm{mm}$ and a density of 10~$\rm{g}/\rm{cm}^3$, contained within a bottle whose geometry is defined in the same way to the dishes in the previous subsection. 
We program the robot to reach for, grasp the bottle, and then pour the particles into another bottle, as depicted in the simulation snapshots in Fig.~\ref{fig:pouring_snapshot}.

This scenario is characterized by a large number of bodies; therefore, both the system DOF and the constraint DOF are large, yet their relationship is relatively sparse compared to previous examples. 
Moreover, the results are expected to maintain a stable grasp and precisely emulate the coupling between a high-gain controlled robot and lightweight particles.
Our primary objective for this scenario is to evaluate the scalability of the solvers. Therefore, we utilize a fixed number of iterations for each solver—5 for CANAL and 100 for SubADMM and PGS, then observe the average residual and computation time over the task execution while varying the number of particles in the bottle (9, 18, 27, 36, 45).

The results are depicted in Fig.~\ref{fig:pouring_scalability}.
As shown in the plots, the increase in computation time with respect to the particle number (therefore, system DOF) is ordered as SubADMM $<$ PGS $<$ CANAL. 
This result aligns with the theoretical properties, as the alternating steps of SubADMM scale at least linearly with the number of subsystems and constraints. 
Note that this scalability could potentially be reduced with the adoption of more advanced parallelization hardware architectures, although we remain this as a work for future implementation.
In the case of PGS, while the computation time for each Gauss-Seidel iteration step increases near linearly, the computation required to establish a contact relation in the dual space (related to the Delassus operator) increases superlinearly, making the overall exponent over SubADMM.  
CANAL scales superlinearly mainly due to the factorization of the Hessian matrix \eqref{eq:canal_hessian}. However, the exponent is still significantly lower than $3$, which is typical for dense matrix factorization, because we leverage the sparse structure of the contact Jacobian.

In terms of accuracy, CANAL outperforms both SubADMM and PGS, aligning with theoretical expectations and previous results. 
Within $5$ iterations, the solver achieves a residual under $10^{-5}$. Additionally, SubADMM consistently demonstrates better accuracy than PGS, achieving approximately 0.1 times the residual of PGS.
We find that using PGS, attempting a strong grasp on the bottle can lead to significant instability in the simulation. For all solvers, the residual is not significantly affected by the number of particles.

\subsection{Ablation Studies}

Lastly, we conduct ablation studies to validate the effectiveness of the technical components presented in this article.

\subsubsection{CANAL}

\begin{table}[t]
\centering
\caption{Comparison results of leveraging a cascaded Newton structure versus directly applying standard semismooth Newton solvers in an AL-based multi-contact solver.}
\renewcommand{\arraystretch}{1.5}{
\begin{tabular}{|c|c|c|c|}
\hline
& CANAL & TRDogleg & DampedSN  \\
\hline
Inner iter. & 30.55 & 118.8 & 246.9  \\
\hline 
Failure[$\%$] &  0  & 1.2 & 1.1  \\
\hline 
Residual[-$\log$] & 8.2231 & 8.5693 & 8.5644 \\
\hline
\end{tabular}
}
\label{table:ablation_canal}
\end{table}

First, we compare a cascaded Newton structure in CANAL with performing augmented Lagrangian by directly solving \eqref{eq:semismooth_eq} using standard semismooth Newton methods.
For baselines, we implement two algorithms: trust-region dogleg algorithm \cite{nocedal1999numerical}, which is a widely standard method for nonlinear equation solver, and damped semismooth Newton algorithms based on backtracking line search on merit function \cite{de1996semismooth}.
For both methods, Jacobian of $r(\hat{v})$ is derived similarly with \eqref{eq:canal_hessian_33} and \eqref{eq:canal_hessian}. 
However, since the operator $T$ is a strict operator here, the Jacobian is not guaranteed to be symmetric positive definite. Therefore, solving the linear problem may require additional computational effort in practice, compared to the Newton step computation in CANAL \eqref{eq:canal_newtonstep}.

For the test, we utilize the same environment and cases described in the single-step test for the dish piling. The results are shown in Table~\ref{table:ablation_canal}.  
Here, we perform 10 AL loops for each test case, and measure the number of inner iterations (i.e., Newton steps) throughout the entire AL loop, the failure rate (if the inner loop fails to find the surrogate problem solution within the desired tolerance), and the residuals of the resulting solutions.

As indicated in the table, the number of inner iterations is significantly lower for CANAL. This reduction is due to its cascaded structure, which can transform solving the surrogate problem into solving a convex optimization problem. Consequently, this structure allows for a convergence guarantee to the global minimum with exact line search.
However, both TRDogleg and DampedSN often gets stuck in a zone where it cannot effectively reduce the merit function, leading to requirement of more inner iteration steps and occasional failure. TRDogleg relatively performs better than DampedSN, with half the required inner iterations. 
Comparing the residuals after 10 iterations, CANAL exhibits similarly low residuals compared to the two baselines. This indicates that updating $\tilde{e}_i$ in every AL iteration within the cascaded structure does not substantially degrade the convergence properties of the entire algorithm.

\subsubsection{SubADMM}

Next, we compare our SubADMM with ADMM implementations based on more standard splitting strategies. For baselines, we implement the following methods: 1) ADMM without subsystem-based splitting (\texttt{No Sub}), as described earlier in Sec.~\ref{sec:subadmm}, which requires solving a full system size linear equation \eqref{eq:vanilla_admm_linear}, and 2) ADMM that utilizes subsystem-based splitting but does not employ body-based splitting \eqref{eq:jacobian_body_splitting} (\texttt{No Body}), which is equivalent to the algorithm in our prior work \cite{lee2023modular}.

\begin{figure}[t]
\centering
\includegraphics[width=8.4cm]{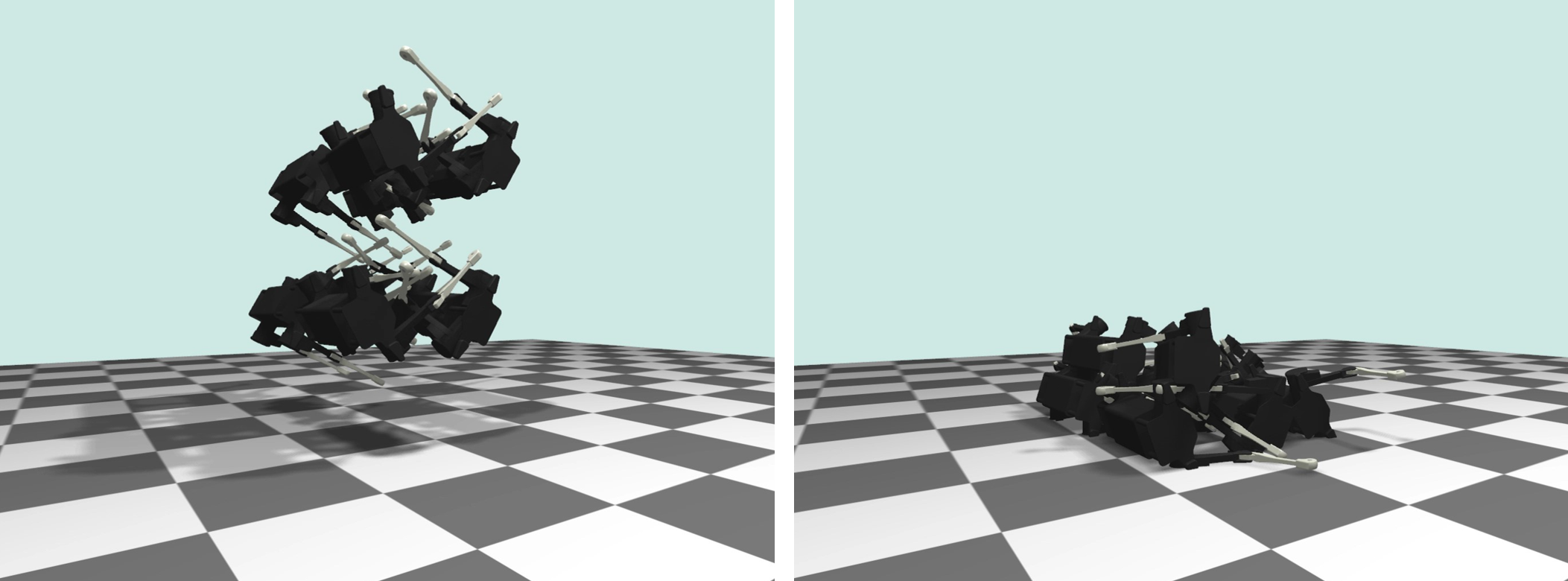}
\caption{Snapshots of a multiple Laikago dropping simulation. The system is divided into 8 subsystems, each interpreted as a kinematic tree.}
\label{fig:laikagodrop_snapshot}
\end{figure}

\begin{table}[t]
\centering
\caption{Comparison Results of ADMM with subsystem-based splitting versus more standard splitting strategies.}
\renewcommand{\arraystretch}{1.5}{
\begin{tabular}{|c|c|c|c|}
\hline
& SubADMM & No Sub & No Body  \\
\hline
Time[ms] & 0.435 & 1.21 & 0.638  \\
\hline 
Residual[-$\log$] &  4.8802  &  4.8817 & 4.8562  \\
\hline 
\end{tabular}
}
\label{table:ablation_subadmm}
\end{table}

For the test, we set up an environment where multiple quadrupedal robots Laikago, are dropped onto the floor, as depicted in Fig.~\ref{fig:laikagodrop_snapshot}.
In this scenario, each robot possesses 18 DOF, resulting in a total system DOF of 144. For the collision geometry, we approximate the trunk, legs, and feet using simple primitives such as boxes and spheres. Typically, $80-100$ contacts are generated during the simulation.  
We use fixed 100 iterations for SubADMM and other baselines.

The comparison results are shown in Table~\ref{table:ablation_subadmm}. 
As demonstrated, the ablation studies highlight the advantages of the proposed techniques. 
SubADMM achieves the shortest computation time due to its ability to utilize efficient matrix assembly, parallelized matrix factorization and solving.
\texttt{No Sub} takes the longest computation time, requiring about $2.78$ times more than SubADMM primarily due to its need for whole system size matrix factorization. 
Compared to \texttt{No Body}, SubADMM is about $1.46$ times faster, as \texttt{No Body} cannot exploit efficient submatrix factorization based on the structure given in \eqref{eq:submatrix_structure}. 
Such differences could become more pronounced as the overall system dimension increases, or through the employment of advanced code parallelization techniques in SubADMM.
For example, in 27 Laikago dropping simulation, we find SubADMM is $3.63$ times faster than \texttt{No Sub}.
The residuals for all three solvers are similar, as they share similar theoretical convergence properties of ADMM.

\section{Discussions and Concluding Remarks} 
\label{sec:discussremark}

In this article, we introduce two multi-contact solver algorithms, CANAL and SubADMM, based on variations of the augmented Lagrangian method. Our formulation extends the theory of AL to handle multi-contact NCP by iteratively solving surrogate problems and subsequently updating primal and dual variables. 
In CANAL, we variate this AL-based structure into a cascaded form of convex optimization, which can be solved by exact Newton steps, thereby ensuring accurate and robust simulation results. 
In SubADMM, we employ the concept of ADMM to enable an alternating solution approach to the surrogate problem. Here, we propose a novel subsystem-based variable splitting method, which not only achieves a parallelizable structure but also preserves the sparsity pattern of the submatrix, significantly improving efficiency.
The examples demonstrate their effectiveness in various robotic simulations characterized by intensive contact formation and stiff interactions, and also illustrate the trade-offs between CANAL, SubADMM, and other existing methods.

A variety of future research areas remain open within the presented framework. From an algorithmic perspective, the CANAL algorithm could be tested with other convex optimization methods, such as the conjugate gradient \cite{fletcher1964function} or accelerated projected gradient\cite{lee2022large}, for example. SubADMM could be enhanced with various strategies to improve its convergence, including the acceleration of fixed-point iterations \cite{themelis2019supermann} and advanced penalty parameter update schemes. Although our initial trials observed that the application of these schemes could degrade the robustness of the simulation, the development of a solid methodology still remains an open question.
From an implementation perspective, several component of the current framework can be improved. For instance, tailored factorization based on the branch-induced sparsity structure could be adopted for CANAL. 
Additionally, a GPU implementation for SubADMM could fully exploit its parallelizable nature.

We believe that the algorithms presented in this article can be further employed in the development of other model-based solvers for robotic applications. For example, our problem described in \eqref{eq:prob_original}, when formulated without friction, becomes equivalent to a quadratic programming problem, which is commonly encountered in motion primitives, planning, and model predictive control of robotic systems. Moreover, a forward simulation solver can be coupled with the differentiation of the results, utilized to address diverse inverse problems involving contact \cite{howell2022dojo}. Our highly accurate solutions are particularly beneficial from this perspective.

Finally, while the focus of this article is on contact solvers, we also believe that contact modeling is a crucial aspect of robotic simulation. This includes the representation of geometry, definition of contact features, time stepping, friction modeling, and often linked with the contact solver. In this regard, investigating how our AL-based solver can be effectively integrated with various aspects, including continuous collision detection \cite{li2020incremental}, contact fields \cite{castro2023theory}, temporal position updates \cite{macklin2019small}, anisotropic friction, and lubrication \cite{ludema2018friction}, will be a valuable topic to explore.

\appendices

\section{Composite Rigid Body Algorithm for Subsystem Matrices} \label{appendix:crba}

Contact Jacobian with respect to spatial velocity of the body can be written as follows:
\begin{align} \label{eq:contact_jacobian}
    J_{i,b_k} = 
    R_i\begin{bmatrix}
        I_{3\times 3} & -[p_{i}] 
    \end{bmatrix}
\end{align}
where $R_i\in \mathrm{SO}(3)$ is the contact frame generated through the contact normal and $p_i$ is the global position of the contact point.
Based on \eqref{eq:contact_jacobian}, \eqref{eq:submatrix_body} can be written as follows:
\begin{align*}
H_{b_k} 
&= M_{b_k} + \sum_i \beta J_{i,b_k}^TJ_{i,b_k} \\
&= \begin{bmatrix} 
m_{b_k} I_{3\times 3} & -m_{b_k}[c_{b_k}] \\
    m_{b_k}[c_{b_k}] & I_{b_k}-m_{b_k}[c_{b_k}]^2
\end{bmatrix} 
+ \sum_i \beta \begin{bmatrix}
    I_{3\times 3} & -[p_{i}] \\
    [p_{i}] & - [p_{i}]^2
\end{bmatrix}
\end{align*}
where $m_{b_k}$ is the mass, $I_{b_k}$ is the moment of inertia, $c_{b_k}$ is the global center of mass position of the body.
It can be easily verified that $H_{b_k}$ shares the same fill-in structure as $M_{b_k}$, with only 10 elements required for matrix storage. Therefore, as in the original composite rigid body algorithm, addition and multiplication (of spatial transformation matrix) operations can be performed with equal efficiency.

\section{Signed Distance Function of Dish Module} \label{appendix:dish_sdf}

\begin{figure}[t]
\centering
\includegraphics[width=7.0cm]{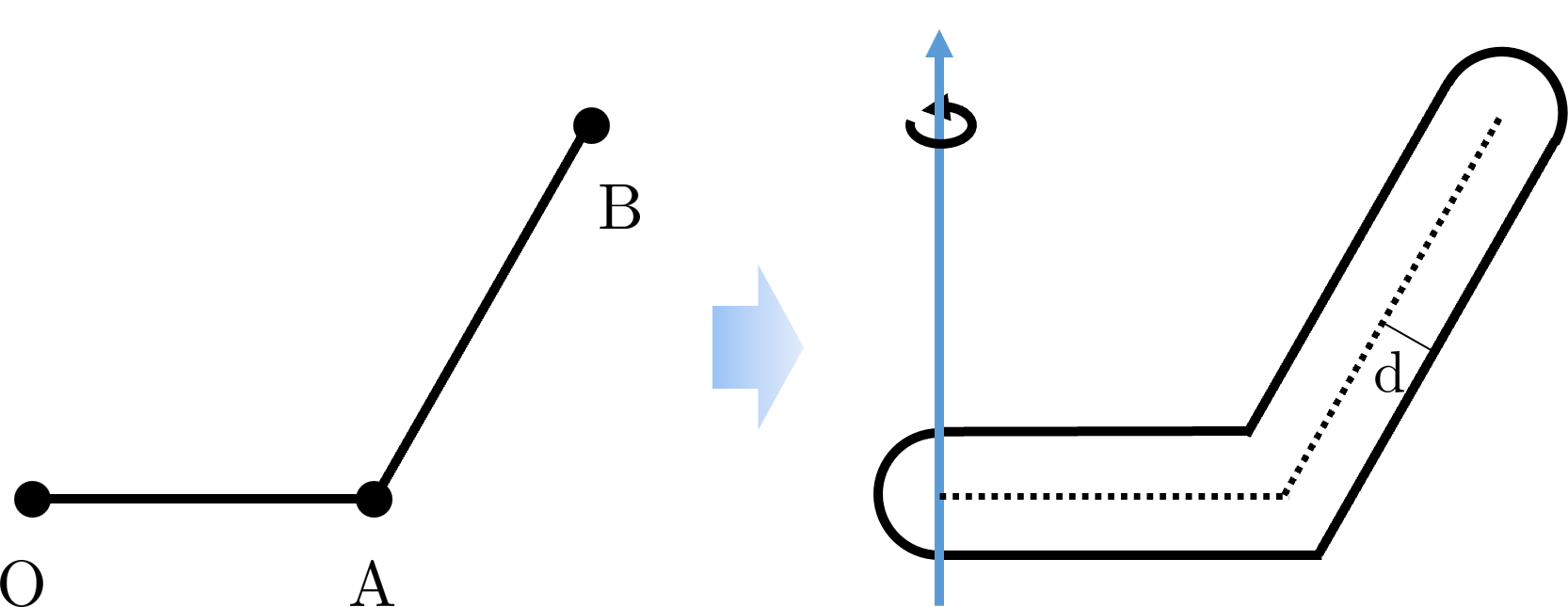}
\caption{Illustrations for dish signed distance function module generation.}
\label{fig:dish_sdf}
\end{figure}

In our experiments, signed distance function of dishes are defined as follows:
\begin{align*}
    \text{SDF}_{\text{3d}}(p) = \text{SDF}_{\text{2d}}\left(\begin{bmatrix}
        \sqrt{p_1^2+p_2^2} & p_3
    \end{bmatrix}\right)
\end{align*}
which is indeed revolution of following 2D signed distance function (see also Fig.~\ref{fig:dish_sdf}):
\begin{align} \label{eq:dish_sdf_2d}
    \text{SDF}_{\text{2d}}(p) = \min\left(\text{dist}(p,\overline{OA}),\text{dist}(p,\overline{AB})\right)-d
\end{align}
where $A,B$ are points on the plane, and $d$ is a thickness for padding.
By adjusting $A,B$ and $d$ in \eqref{eq:dish_sdf_2d}, we can generate diverse range of dishes.
Additionally, the derivative of \eqref{eq:dish_sdf_2d} can be computed analytically, and thus it can be utilized in the collision detection process.

\bibliographystyle{unsrt}
\bibliography{reference}

\end{document}